\documentclass{article} % For LaTeX2e
\usepackage{iclr2020_conference,times}

\usepackage{amsmath,amsfonts,bm}

\def\eqref#1{equation~\ref{#1}}
\def\1{\bm{1}}

\DeclareMathAlphabet{\mathsfit}{\encodingdefault}{\sfdefault}{m}{sl}
\SetMathAlphabet{\mathsfit}{bold}{\encodingdefault}{\sfdefault}{bx}{n}

\usepackage{hyperref}
\usepackage{url}
\usepackage[utf8]{inputenc} % allow utf-8 input
\usepackage[T1]{fontenc}    % use 8-bit T1 fonts
\usepackage{booktabs}       % professional-quality tables
\usepackage{amsfonts,amsthm}       % blackboard math symbols
\usepackage{nicefrac}       % compact symbols for 1/2, etc.
\usepackage{microtype}      % microtypography
\usepackage{bm} 
\usepackage{graphicx}
\usepackage{color}

\usepackage{caption}
\usepackage{subfigure}
\usepackage{graphbox,multirow}
\usepackage[para,online,flushleft]{threeparttable}
\usepackage{enumitem}
\usepackage{amsmath,mathtools}
\usepackage{amssymb}
\usepackage{makecell}

\newtheorem{theorem}{Theorem}

\newtheorem{lemma}[theorem]{Lemma}

\newtheorem{proposition}{Proposition}

\title{Bridging Mode Connectivity in Loss \\ Landscapes and Adversarial Robustness}

\author{  \textbf{Pu Zhao\textsuperscript{\rm 1}, Pin-Yu Chen\textsuperscript{\rm 2} , Payel Das\textsuperscript{\rm 2}, Karthikeyan Natesan Ramamurthy\textsuperscript{\rm 2}, Xue Lin\textsuperscript{\rm 1}}\\ 
\textsuperscript{\rm 1}Northeastern University, Boston, MA 02115\\ 
\textsuperscript{\rm 2}IBM Research, Yorktown Heights, NY 10598\\
zhao.pu@husky.neu.edu, pin-yu.chen@ibm.com, daspa@us.ibm.com, \\
knatesa@us.ibm.com, xue.lin@northeastern.edu 
}

\iclrfinalcopy % Uncomment for camera-ready version, but NOT for submission.
\begin{document}

\maketitle

\begin{abstract}
 Mode connectivity 
  provides novel geometric insights on analyzing loss landscapes and
  enables building high-accuracy pathways between well-trained neural networks. In this work, we propose to employ mode connectivity in loss landscapes to study the adversarial robustness of deep neural networks, and provide novel methods for improving this robustness. Our experiments cover various types of adversarial attacks applied to different network architectures and datasets. When network models are tampered with backdoor or error-injection attacks, our results demonstrate that the path connection learned using limited amount of bonafide data can effectively mitigate adversarial effects while maintaining the original accuracy on clean data. Therefore, mode connectivity provides users with the power to repair backdoored or error-injected models. We also use mode connectivity to investigate the loss landscapes of regular and robust models against evasion attacks. Experiments show that there exists a barrier in adversarial robustness loss on the path connecting regular and adversarially-trained models. A high correlation is observed between the adversarial robustness loss and the largest eigenvalue of the input Hessian matrix, for which theoretical justifications are provided. Our results suggest that mode connectivity offers a holistic tool and practical means for evaluating and improving adversarial robustness\footnote{The code is available at \url{https://github.com/IBM/model-sanitization}}.
\end{abstract}

\section{Introduction}
\label{sec_intro}

Recent studies on mode connectivity show that two independently trained deep neural network (DNN) models with the same architecture and loss function can be connected on their loss landscape using a  high-accuracy/low-loss path characterized by a simple curve \citep{garipov2018loss,gotmare2018using,draxler2018essentially}. This insight on the loss landscape geometry provides us with easy access to a large number of similar-performing models on the low-loss path between two given models, and  \citet{garipov2018loss} use this to devise a new model ensembling method. Another line of recent research reveals interesting geometric properties relating to  adversarial robustness of DNNs \citep{fawzi2017robustness,fawzi2018empirical, wang2018towards, yu2018interpreting}. An adversarial data or model is defined to be one that is close to a bonafide data or model in some space, but exhibits unwanted or malicious behavior. Motivated by these geometric perspectives, in this study, we propose to employ mode connectivity to study and improve adversarial robustness of DNNs against different types of threats.

A DNN can be possibly tampered by an adversary during different phases in its life cycle. For example, during the training phase, the training data can be corrupted with a designated trigger pattern associated with a target label to implant a backdoor for trojan attack on DNNs \citep{BadNet_Access,liu2017trojaning}. During the inference phase when a trained model is deployed for task-solving, prediction-evasive attacks are plausible \citep{biggio2018wild,goodfellow2014explaining,zhao2018admm}, even when the model internal details are unknown to an attacker \citep{chen2017zoo,ilyas2018black,zhao2019design}. In this research, we will demonstrate that by using mode connectivity in loss landscapes, we can repair backdoored or error-injected DNNs. We also show that mode connectivity analysis reveals the existence of a robustness loss barrier on the path connecting regular and adversarially-trained models.

%We will also use mode connectivity to provide a holistic analysis of both standard and robust loss landscapes. Specifically,

We motivate the novelty and benefit of using mode connectivity for mitigating training-phase adversarial threats through the following practical scenario: as training DNNs is both time- and resource-consuming, it has become a common trend for users to leverage pre-trained models released in the public domain\footnote{For example, the Model Zoo project: \url{https://modelzoo.co}}. Users may then perform model fine-tuning or transfer learning with a small set of bonafide data that they have. However, publicly available pre-trained models may carry an unknown but significant risk of tampering by an adversary. It can also be challenging to detect this tampering, as in the case of a backdoor attack\footnote{See the recent call for proposals on Trojans in AI announced by IARPA: \url{https://www.iarpa.gov/index.php/research-programs/trojai/trojai-baa}}, since a backdoored model will behave like a regular model in the absence of the embedded trigger. Therefore, it is practically helpful to provide tools to users who wish to utilize pre-trained models while mitigating such adversarial threats. We show that our proposed method using mode connectivity with limited amount of bonafide data can repair backdoored or error-injected DNNs, while greatly countering their adversarial effects. 
%Experiments show that our proposed approach significantly outperforms several baselines including fine-tuning, training from scratch, pruning, and random weight perturbations. 

Our main contributions are summarized as follows:
\begin{itemize}[leftmargin=*]
    \item For backdoor and error-injection attacks, we show that the path trained using limited bonafide data connecting two tampered models can be used to repair and redeem the attacked models, thereby resulting in high-accuracy and low-risk models. The performance of mode connectivity is significantly better than several baselines including fine-tuning, training from scratch, pruning, and random weight perturbations. \textcolor{black}{We also provide technical explanations for the effectiveness of our path connection method based on model weight space exploration and similarity analysis of input gradients for clean and tampered data.}
    
    \item For evasion attacks, we use mode connectivity to study standard and adversarial-robustness loss landscapes. We find that between a regular and an adversarially-trained model, training a path with standard loss reveals no barrier, whereas the robustness loss on the same path reveals a barrier. This insight provides a geometric interpretation of the ``no free lunch'' hypothesis in adversarial robustness \citep{tsipras2019robustness,dohmatob2018limitations,bubeck2018adversarial}. We also provide technical explanations for the high correlation observed between the robustness loss and the largest eigenvalue of the input Hessian matrix on the path.
    
    \item Our experimental results on different DNN architectures (ResNet and VGG) and datasets (CIFAR-10 and SVHN) corroborate the effectiveness of using mode connectivity in loss landscapes to understand and improve adversarial robustness. \textcolor{black}{We also show that our path connection is resilient to the considered adaptive attacks that are aware of our defense.}
    To the best of our knowledge, this is the first work that proposes using mode connectivity approaches for adversarial robustness.
\end{itemize}

\section{Background and Related Work}

\subsection{Mode Connectivity in Loss Landscapes}
Let $w_1$ and $w_2$ be two sets of model weights corresponding to two neural networks independently trained by minimizing any user-specified loss $l(w)$, such as the cross-entropy loss. 
Moreover, let ${\phi _\theta }(t)$ with $t \in [0,1]$  be a continuous piece-wise smooth parametric curve, with parameters $\theta$, such that its two ends are ${\phi _\theta }(0)=w_1$ and ${\phi _\theta }(1)=w_2$.

To find a high-accuracy path  between $w_1$ and $w_2$, it is proposed to find the parameters $\theta$ that minimize the expectation over a uniform distribution on the curve \citep{garipov2018loss},
\begin{equation}
L({\bf{\theta }}) = {E_{t \sim {q_\theta }(t)}}\left[ {l({\phi _\theta }(t))} \right]
\end{equation}
where ${q_\theta }(t)$ is the distribution for sampling the models on the path indexed by $t$.

Since ${q_\theta }(t)$ depends on $\theta$, in order to render the training of high-accuracy path connection more computationally tractable, \citep{garipov2018loss,gotmare2018using} proposed to instead use the following loss term,
%It's hard to compute the loss since ${q_\theta }(t)$ depends on $\theta$.  Therefore a more computationally tractable loss is like the following,
\begin{equation}
\label{eqn_path_loss}
L({\bf{\theta }}) = {E_{t \sim U(0,1)}}\left[ {l({\phi _\theta }(t))} \right]
\end{equation}
where $U(0,1)$ is the uniform distribution on $[0,1]$.  

The following functions are commonly used for characterizing the parametric curve function ${\phi _\theta }(t)$.
\textbf{Polygonal chain \citep{gomes2012computer}.} 
%The simplest parametric curve we consider is the polygonal chain. 
The two trained networks $w_1$  and $w_2$ serve as the endpoints of the chain and the bends of the chain are parameterized by $\theta$. For instance, the case of a chain with one bend is
\begin{equation} \label{equ_poly}
{\phi _\theta }(t) = \left\{ {\begin{array}{*{20}{c}}
{2\left( {t{\bf{\theta }} + \left( {0.5 - t} \right){{\bf{\omega }}_1}} \right),}&{0 \le t \le 0.5}\\
{2\left( {\left( {t - 0.5} \right){{\bf{\omega }}_2} + \left( {1 - t} \right){\bf{\theta }}} \right),}&{0.5 \le t \le 1.}
\end{array}} \right.
\end{equation}

\textbf{Bezier curve \citep{rida2012bernstein}.}  A Bezier curve  provides a convenient parametrization of smoothness on the paths connecting endpoints. For instance, a quadratic Bezier curve with endpoints $w_1$  and $w_2$  is given by
\begin{equation}
\label{eqn_Bezier_quad}
{\phi _\theta }(t) = \begin{array}{*{20}{c}}
{{{\left( {1 - t} \right)}^2}{{\bf{\omega }}_1} + 2t\left( {1 - t} \right){\bf{\theta }} + {t^2}{{\bf{\omega }}_2},}&{0 \le t \le 1.}
\end{array}
\end{equation}

It is worth noting that, while current research on mode connectivity mainly focuses on generalization analysis \citep{garipov2018loss,gotmare2018using,draxler2018essentially,wang2018identifying} and
has found remarkable applications such as fast model ensembling \citep{garipov2018loss}, our results show that its implication on adversarial robustness through the lens of loss landscape analysis is a promising, yet largely unexplored, research direction. \citet{yu2018interpreting} scratched the surface but focused on interpreting decision surface of input space and only considered evasion attacks.

%\nrk{Mode connectivity used for understanding loss landscapes, we take it to adversarial defense.}

\subsection{Backdoor,  Evasion, and Error-Injection Adversarial Attacks}

\textbf{Backdoor attack.}
Backdoor attack on DNNs is often accomplished by designing a designated trigger pattern with a target label implanted to a subset of training data, which is a specific form of data poisoning  \citep{Biggio2012poison,shafahi2018poison,jagielski2018manipulating}. A backdoored model trained on the corrupted data will output the target label for any data input with the trigger; and it will behave as a normal model when the trigger is absent. For mitigating backdoor attacks,  majority of research focuses on backdoor detection or filtering anomalous data samples from training data for re-training \citep{chen2018detecting,wang2019neural,tran2018spectral}, while our aim is to repair backdoored models for models using mode connectivity and limited amount of bonafide data.

\textbf{Evasion attack.}
Evasion attack is a type of inference-phase adversarial threat that generates adversarial examples by mounting slight modification on a benign data sample to manipulate model prediction \citep{biggio2018wild,wang2018defensive}.
For image classification models, evasion attack can be accomplished by 
adding imperceptible noises to natural images and resulting in misclassification \citep{goodfellow2014explaining,carlini2017towards,xu2018structured,8646578}. Different from training-phase attacks, evasion attack does not assume access to training data. Moreover, it can be executed even when the model details are unknown to an adversary, via black-box or transfer attacks \citep{papernot2017practical,chen2017zoo,zhao2020zongd}.

\textbf{Error-injection attack.}
%The adversarial attacks can be reviewed from the aspects of perturbing the inputs and perturbing the DNN parameters. Evasion attacks generate adversarial examples to fool DNNs by perturbing the legitimate inputs.
%Basically, an adversarial example is produced by adding human-imperceptible distortions onto a legitimate image, such that the adversarial example will be classified by the DNN as a target (wrong) label \citep{goodfellow2014explaining,carlini2017towards}. 
Different from attacks modifying data inputs, error-injection attack injects errors to model weights at the inference phase and aims to cause misclassification of certain input samples \citep{liu2017fault,pu2019fault}. At the hardware level of a deployed machine learning system, it can be made plausible via laser beam \citep{barenghi2012fault} and row hammer  \citep{van2016drammer} to change or flip the logic values of the corresponding bits and thus modifying the model parameters saved in memory.

\section{Main Results}
Here we report the experimental results, provide technical explanations, and elucidate the effectiveness of using mode connectivity for studying and enhancing adversarial robustness in three representative themes: (i) backdoor attack; (ii) error-injection attack; and (iii) evasion attack. Our experiments were conducted on different network architectures (VGG and ResNet) and datasets (CIFAR-10 and SVHN). The details on experiment setups are given in Appendix \ref{appen_network}.
When connecting models, we use the cross entropy loss and the quadratic Bezier curve as described in (\ref{eqn_Bezier_quad}). In what follows, we begin by illustrating the problem setups bridging mode connectivity and adversarial robustness, summarizing the results of high-accuracy (low-loss) pathways between untampered models for reference, and then delving into detailed discussions. Depending on the context, we will use the terms error rate and accuracy on clean/adversarial samples interchangeably. The error rate of adversarial samples is equivalent to their attack failure rate  as well as 100\%- attack accuracy.

\subsection{Problem Setup and Mode Connection between Untampered Models}
\label{subsec_connect_untampered}
\textbf{Problem setup for backdoor and error-injection attacks.}
We consider the practical scenario as motivated in Section \ref{sec_intro}, where a user has two potentially tampered models and a limited number of bonafide data at hand. The tampered models behave normally as untampered ones on non-triggered/non-targeted inputs so the user aims to fully exploit the model power while alleviating adversarial effects. The problem setup applies to the case of one tampered model, where we use the bonafide data to train a fine-tuned model and then connect the given and the fine-tuned models.

\textbf{Problem setup for evasion attack.} For gaining deeper understanding on evasion attacks, we consider the scenario where a user has access to the entire training dataset and aims to study the behavior of the models on the path connecting two independently trained models in terms of standard and robust loss landscapes, including model pairs selected from regular and adversarially-trained models.

\textbf{Regular path connection between untampered models.}  Figure \ref{fig: regular_cifar_VGG} shows the cross entropy loss and training/test error rate of models on the path connecting untampered models. The untampered models are independently trained using the entire training data. While prior results have demonstrated high-accuracy path connection using the entire training data \citep{garipov2018loss,gotmare2018using,draxler2018essentially}, our path connection is trained using different portion of the original test data corresponding to the scenario of limited amount of bonafide data. Notably, when connecting two DNNs, a small number of clean data is capable of finding models with good performance. For example, path connection using merely 1000/2500 CIFAR-10 samples only reduces the test accuracy (on other 5000 samples) of VGG16 models by at most 10\%/5\% when compared to the  well-trained models (those at $t=0$ and $t=1$), respectively. In addition, regardless of the data size used for path connection, the model having the worst performance is usually located around the point $t=0.5$, as it is geometrically the farthest model from the two end models on the path.

\begin{figure}[t]    
\hspace{-5.5mm}
 \centering
\begin{tabular}{p{1.9in}p{1.9in}p{1.2in}}
\parbox{1.9in}{\centering  Inference on training set}
 & \parbox{1.9in}{\centering  Inference on test set}
  & \parbox{1.2in}{\centering  Legend}\\
  \includegraphics[align=c,width=1.9in]{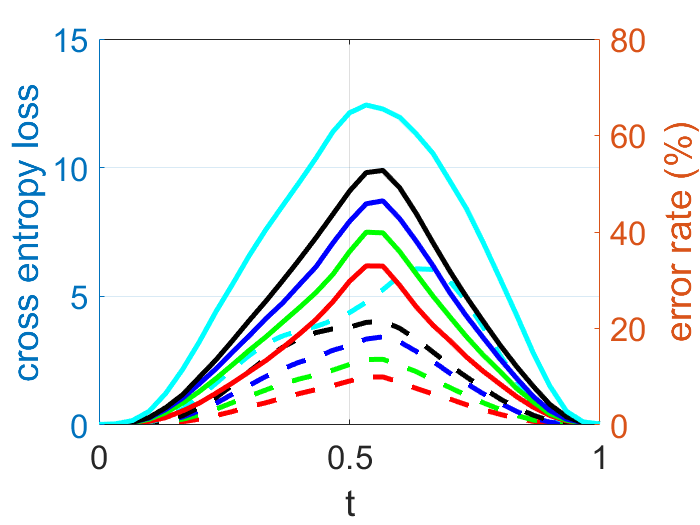}  &
\includegraphics[align=c,width=1.9in]{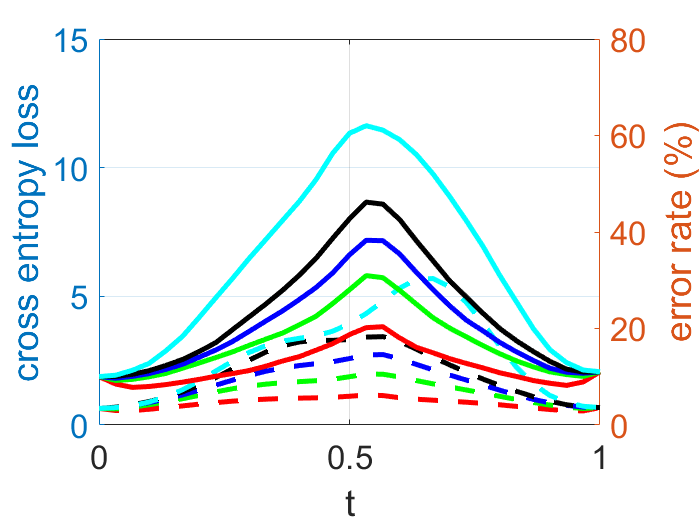} & 
\includegraphics[align=c,width=1.2in]{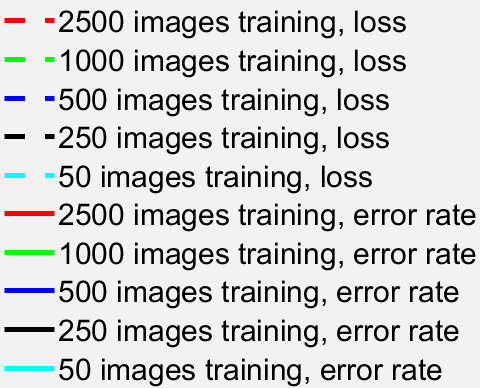} 
\end{tabular}
\vspace{-2mm}
\caption{Loss and error rate on the path connecting two untampered VGG models trained on CIFAR-10. The path connection is trained using different settings as indicated by the curve colors. The results using SVHN and ResNet are given in Appendix \ref{appen_untampered}. The inference results on test set  are evaluated  using \textcolor{black}{5000} samples, which are separate from what are used for path connection.  }
\label{fig: regular_cifar_VGG}
%\vspace{-2mm}
\end{figure}

\begin{table}[tb]
\begin{center}
\caption{Error rate of backdoored models. The error rate of clean/backdoored samples means  standard-test-error/attack-failure-rate, respectively.
The results are evaluated on 5000 non-overlapping clean/triggered images selected from the test set. For reference, the test errors of clean images on untampered models are 12\% for CIFAR-10 (VGG), and  4\% for SVHN (ResNet), respectively. }
\label{tab: backdoor_clean_trigger}
\scalebox{0.82}{
\begin{threeparttable}
\begin{tabular}{c|c|c|c|c|c}
\toprule[1pt]
 & Backdoor attacks  &    \multicolumn{2}{c|}{Single-target attack} &    \multicolumn{2}{c}{All-targets attack} \\
 \hline
& Dataset & CIFAR-10 (VGG)  &  SVHN (ResNet)  & CIFAR-10 (VGG)  & SVHN (ResNet)   \\
\midrule[1pt]
\multirow{2}{*}{Model ($t=0$)}  & Clean images &  15\% & 5.4\% & 14.2\% & 6.1\%  \\
& Triggered images &  0.07\% & 0.22\% & 12.9\%& 8.3\% \\
\midrule[1pt]
\multirow{2}{*}{Model ($t=1$)} & Clean images & 13\% & 7.7\% & 19\%  & 7.5\% \\
& Triggered images & 2\%  & 0.17\% & 13.6\% & 9.2\% \\
\bottomrule[1pt]
\end{tabular}
\end{threeparttable}
}
\end{center}
\vspace{-4mm}
\end{table}

\subsection{Mitigating and Repairing Backdoored Models}
\label{subsec_backdoor}

\textbf{Attack implementation.} We follow the procedures in \citep{BadNet_Access} to implement backdoor attacks and obtain two backdoored models trained on the same poisoned training data. The trigger pattern is placed at the right-bottom  of the poisoned images as shown in Appendix \ref{appen_backdoor}. Specifically, 10\% of the training data are poisoned by inserting the trigger and changing the original correct labels to the target label(s). Here we investigate two kinds of backdoor attacks: (a) single-target attack which sets the target label $T$ to a specific label (we choose $T=$ class 1); and (b) all-targets attack where the target label $T$ is set to  the original label $i$ plus 1 and then modulo 9, i.e., $T=i+1 (\textnormal{mod}~ 9)$.
Their performance on clean (untriggered) and triggered data samples are given in Table \ref{tab: backdoor_clean_trigger}, and the prediction errors of triggered images relative to the true labels are given in Appendix \ref{appen_true_error}.
The backdoored models have similar performance on clean data as untampered models but will indeed misclassify majority of triggered samples. Comparing to single-target attack,
all-targets attack is more difficult and  has a  higher attack failure rate, since the target labels vary with the original labels.

\textbf{Evaluation and analysis.}
We train a path connecting the two backdoored models with limited amount of bonafide data. 
%With the start and end models, we try to find a path to connect them by training with only a part of the test set images. The reason to  use limited number of clean images instead of the whole training or test set  for the connection training is that  the users or defenders may not have large amount of clean images. Otherwise they are able to train their own trust-worthy model with adequate clean images and do not need to rely on outsourced models.  Considering the limited amount of clean images,  it would be  desirable to achieve defense with small number of clean images. 
%Note that the images used to train the path connection and the images used to test the model's accuracy on the path do not overlap. 
As shown in Figure \ref{fig: backdoor_cifar_VGG}, at both path endpoints ($t=\{0,1\}$) the two tampered models attain low error rates on clean data but are also extremely vulnerable to backdoor attacks (low error rate on backdoored samples means high attack success rate). Nonetheless, we find that path connection with limited bonafide data can effectively mitigate backdoor attacks and redeem model power. 
For instance, the models at $t=0.1$ or $t=0.9$ can simultaneously attain similar performance on clean data as the tampered models while greatly reducing the backdoor attack success rate from close to 100\% to nearly 0\%. Moreover, most models on the path (e.g. when $t\in[0.1,0.9]$) exhibit high resilience to backdoor attacks, suggesting mode connection with limited amount of bonafide data can be an effective countermeasure. While having resilient models to backdoor attacks on the path,
we also observe that the amount of bonafide data used for training path has a larger impact on the performance of clean data. Path connection using fewer data samples will yield models with higher error rates on clean data, which is similar to the results of path connection between untampered models discussed in Section \ref{subsec_connect_untampered}. The advantages of redeeming model power using mode connectivity are consistent when evaluated on different network architectures and datasets (see Appendix \ref{appen_backdoor_more}).

\begin{figure}[t]    
\hspace{-5.5mm}
 \centering
\begin{tabular}{p{1.9in}p{1.9in}p{1.2in}}
\parbox{1.9in}{\centering  Single-target backdoor attack}
 & \parbox{1.9in}{\centering  All-targets backdoor attack}
  & \parbox{1.2in}{\centering  Legend} 
  \\
\includegraphics[align=c,width=1.9in]{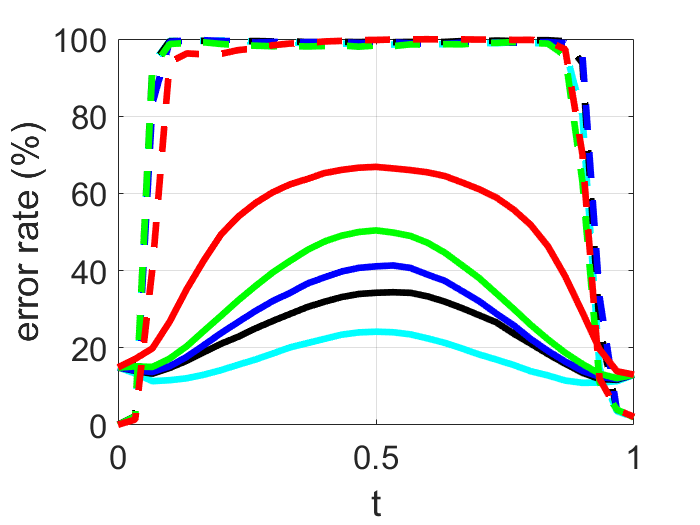}  &
\includegraphics[align=c,width=1.9in]{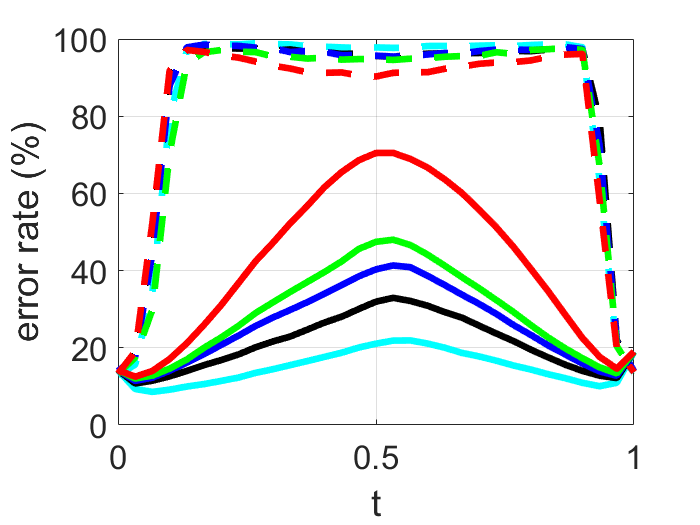} & 
\includegraphics[align=c,width=1.2in]{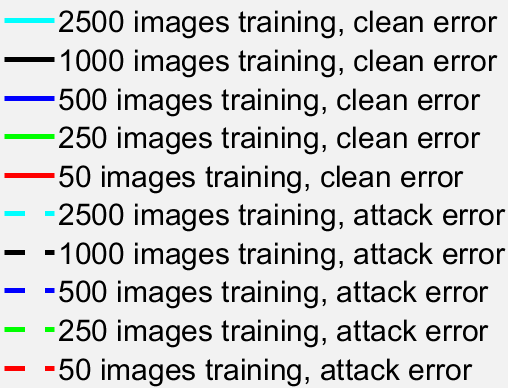} 
\end{tabular}
\vspace{-2mm}
\caption{Error rate against  backdoor attacks on the connection path for CIFAR-10 (VGG). The error rate of clean/backdoored samples means the standard-test-error/attack-failure-rate, respectively.} 
\label{fig: backdoor_cifar_VGG}
\vspace{-4mm}
\end{figure}

\begin{table}[tb]
\begin{center}
\caption{Performance against single-target backdoor attack. The clean/backdoor accuracy means standard-test-accuracy/attack-success-rate, respectively. More results are given in Appendix \ref{appen_backdoor_more}.} 
\label{tab: backdoor_comparison}
\scalebox{0.92}{
\begin{threeparttable}
\begin{tabular}{c|c|c|c|c|c|c|c}
\toprule[1pt]
& & Method / Bonafide data size & 2500 & 1000  & 500  &  250 & 50\\
\midrule[1pt]
\multirow{ 12}{*}{\makecell{CIFAR-10 \\ (VGG)}}   &\multirow{ 6}{*}{ \makecell{Clean \\  Accuracy}}&Path connection $(t=0.1)$ & 88\% &83\% & 80\% &77\% &63\%\\
& &Fine-tune &   84\% &82\% & 78\% &74\% &46\%\\
& &Train from scratch &  50\% &39\% & 31\% &30\% &20\%\\
& & Noisy model ($t=0$) &  21\% &  21\% & 21\% & 21\% & 21\%  \\
& & Noisy model ($t=1$) &  24\%&  24\%&  24\%&  24\%&  24\%  \\
& & Prune  &  88\%  &  85\% & 83\%  & 82\% &  81\% \\
\cline{2-8}
&\multirow{ 6}{*}{ \makecell{Backdoor \\ Accuracy }}&Path connection $(t=0.1)$ & 1.1\% &0.8\% & 1.5\% &3.3\% & 2.5\%\\
&  & Fine-tune &   1.5\% &0.9\% & 0.5\% & 1.9\% & 2.8\%\\
&  & Train from scratch &  0.4\% &0.7\% & 0.3\% &3.2\% &2.1\%\\
& &  Noisy model ($t=0$) &  97\%  &  97\%  &  97\%  &  97\%  &  97\% \\
& &  Noisy model ($t=1$) & 91\% & 91\% & 91\% &91\% & 91\%  \\
& & Prune  &  43\%  &  49\% &  81\% & 79\% & 82\% \\
\midrule[1pt]
\multirow{ 12}{*}{\makecell{SVHN \\ (ResNet)}}   &\multirow{ 6}{*}{ \makecell{Clean \\  Accuracy} }&Path connection $(t=0.2)$ & 96\% &  94\% & 93\% & 89\% & 82\%\\
& &Fine-tune & 96\% &  94\% & 91\% & 89\%  & 76\%\\
& &Train from scratch & 87\%  & 75\% & 61\% & 34\% &  12\% \\
& & Noisy model ($t=0$) & 13\% &  13\% &  13\% &  13\% &  13\%  \\
& & Noisy model ($t=1$) & 11\% & 11\% & 11\% & 11\% & 11\%  \\
& & Prune  &  96\%  & 95\% & 93\%  & 91\% & 89\% \\
\cline{2-8}
&\multirow{ 6}{*}{ \makecell{Backdoor \\ Accuracy }}&Path connection $(t=0.2)$ & 2.5\% & 3\% & 3.6\% & 4.3\%  & 16\% \\
&  & Fine-tune &   14\% & 7\% & 29\% & 63\% &  60\% \\
&  & Train from scratch &  3\% & 3.6\%  & 5\% & 2.2\% & 3.9\% \\
& & Noisy model ($t=0$) & 51\% & 51\% & 51\% & 51\% & 51\%\\
& & Noisy model ($t=1$) &  42\% &  42\% &  42\% &  42\% &  42\% \\
& & Prune  & 80\%  & 90\% &  88\% & 92\% &  94\% \\
\bottomrule[1pt]
\end{tabular}
\end{threeparttable}
}
\end{center}
\vspace{-4mm}
\end{table}

\textbf{Comparison with baselines.}
We compare the performance of mode connectivity against backdoor attacks with the following baseline methods: (i) fine-tuning backdoored models with bonafide data; (ii)
training a new model of the same architecture from scratch with bonafide data; (iii) model weight pruning and then fine-tuning with bonafide data using \citep{li2016pruning}; and (iv) random Gaussian perturbation to the model weights leading to a noisy model. The results are summarized in Table \ref{tab: backdoor_comparison} and their implementation details are given in Appendix \ref{appen_baseline}.
Evaluated on different network architectures and datasets, the path connection method consistently maintains superior accuracy on clean data while simultaneously attaining low attack accuracy over the baseline methods, which can be explained by the ability of finding high-accuracy paths between two models using mode connectivity.
For CIFAR-10 (VGG), even using as few as 50 bonafide samples for path connection, the subsequent model in Table  \ref{tab: backdoor_comparison} still remains 63\% clean accuracy while constraining backdoor accuracy to merely 2.5\%. The best baseline method is fine-tuning, which has similar backdoor accuracy as path connection but attains lower clean accuracy (e.g. 17\% worse than path connection when using 50 bonafide samples). For SVHN (ResNet), the clean accuracy of fine-tuning can be on par with path connection, but its backdoor accuracy is significantly higher than path connection. For example, when trained with 250 samples, they have the same clean accuracy but the backdoor accuracy of fine-tuning is 58.7\% higher than path connection.
Training from scratch does not yield competitive results given limited amount of training data. Noisy models perturbed by adding zero-mean Gaussian noises to the two models are not effective against backdoor attacks and may suffer from low clean accuracy. Pruning gives high clean accuracy but has very little effect on mitigating backdoor accuracy.

\textbf{Extensions.} Our proposal of using mode connectivity to repair backdoor models can be extended to the case when only one tampered model is given. We propose to fine-tune the model using bonafide data and then connect the given model with the fine-tuned model. Similar to the aforementioned findings, path connection can remain good accuracy on clean data while becoming resilient to backdoor attacks. We refer readers to Appendix \ref{appen_one_model} for more details.
In addition, we obtain similar conclusions when the two backdoored models are trained with different poisoned datasets.

\textcolor{black}{
\textbf{Technical Explanations.}
To provide technical explanations for the effectiveness of our proposed path connection method in repairing backdoored models, we run two sets of analysis: (i) model weight space exploration and (ii) data similarity comparison. For (i), we generate 1000 noisy versions of a backdoored model via random Gaussian weight perturbations. We find that they suffer from low clean accuracy and high attack success rate, which suggests that a good model with high-clean-accuracy and low-attack-accuracy is unlikely to be found by chance. We also report the distinct difference between noisy models and models on the path in the weight space to validate the necessity of using our path connection for attack mitigation and model repairing. More details are given in Appendix \ref{appen_noisy_model}. 
For (ii), we run similarity analysis of the input gradients between the end (backdoored) models  and  models on the connection path for both clean data and triggered data. We find that the similarity of triggered data is much lower than that of clean data when the model is further away in the path from the end models, suggesting that our path connection method can neutralize the backdoor effect. More details are given in Appendix \ref{appen_similarity}. The advantage of our path connection method over fine-tuning demonstrates the importance of using the knowledge of mode connectivity for model repairing.
}

\textcolor{black}{
\textbf{Adaptive Attack.}
To justify the robustness of our proposed path connection approach to adaptive attacks,  we consider the advanced attack setting where the attacker knows path connection is used for defense but cannot compromise the bonafide data that are private to an user. 
Furthermore, we allow the attacker to  use the \textit{same} path training loss function as the defender.
To attempt breaking path connection, the attacker trains a compromised path such that every model on this path is a backdoored model and then releases the path-aware tampered models. We show that our approach is still resilient to this adaptive attack. More details are given in Appendix \ref{appen_adaptive}.
}
%\textcolor{red}{one tampered model case and different poisoned datasets} The above path connection are based on two given backdoored models trained with the same poisoned dataset. In the case of only one given backdoored model, the path connection method can still be applied to obtain a model resistant to triggered images. We can first finetune the model and then connect the original model with the finetuned model. The results are shown in the appendix. Besides, if the  given two backdoored models are trained with different poisoned dataset (e.g. different poisoned number of images), the path connection method works as well in this case. Readers can refer to the appendix for more details. 

\subsection{Sanitizing Error-Injected Models}
\label{subsec_injection}
%The adversarial attacks can be reviewed from the aspects of perturbing the inputs and perturbing the DNN parameters. Evasion attacks generate adversarial examples to fool DNNs by perturbing the legitimate inputs.
%Bsically, an adversarial example is produced by adding human-imperceptible distortions onto a legitimate image, such that the adversarial example will be classified by the DNN as a target (wrong) label \citep{goodfellow2014explaining,carlini2017towards}. 
%On the other side, injection attacks \citep{liu2017fault,pu2019fault} try to cause the misclassification of certain input images by modifying the model parameters. In practice, the common techniques including laser beam \citep{barenghi2012fault} and row hammer  \citep{van2016drammer} can change or flip the logic values of the corresponding bits in memory, thus modifying the model parameters saved in memory. In this section, we demonstrate the defense performance against fault injection attacks. 

\textbf{Attack implementation.} 
We adopt the fault sneaking attack \citep{pu2019fault} for injecting errors to model weights. 
Given two untampered and independently trained models, the errors are injected with selected samples as targets such that the tampered models will cause misclassification on targeted inputs and otherwise will behave as untampered models. More details are given in Appendix \ref{appen_backdoor}.

\begin{figure}[t]    
\hspace{-5.5mm}
 \centering
\begin{tabular}{p{1.9in}p{1.9in}p{1.2in}}
\parbox{1.9in}{\centering CIFAR-10 (VGG)}
 & \parbox{1.9in}{\centering  SVHN (ResNet)}
  & \parbox{1.2in}{\centering  Legend}
  \\
\includegraphics[align=c,width=1.9in]{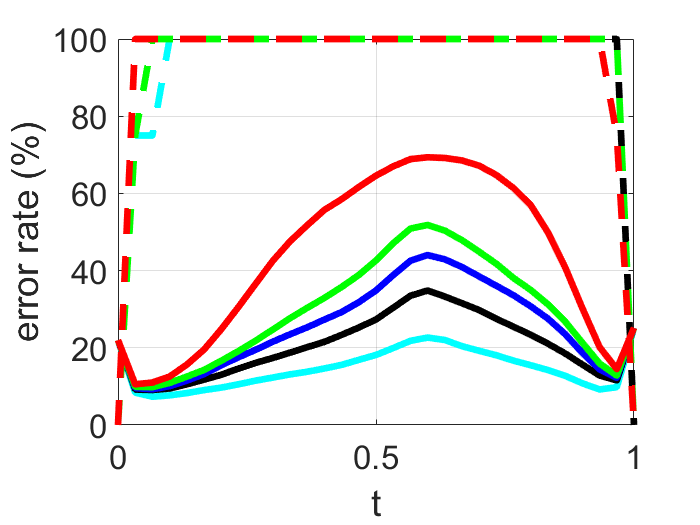}  &
\includegraphics[align=c,width=1.9in]{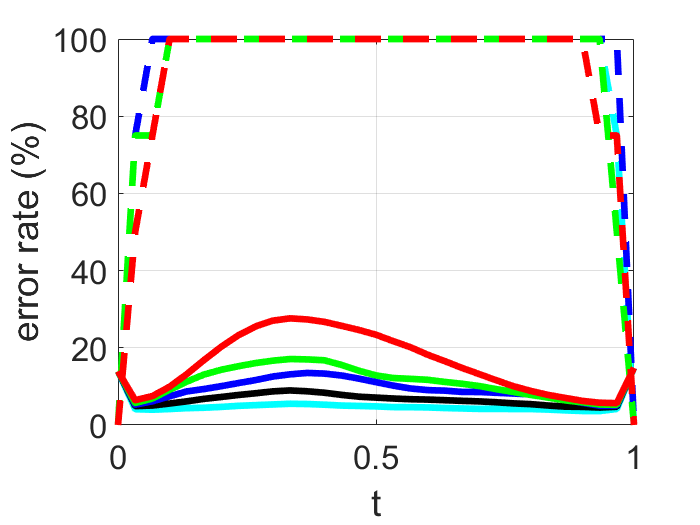} & 
\includegraphics[align=c,width=1.35in]{figs_new/legend.png} 
\end{tabular}
\vspace{-2mm}
\caption{Error rate against error-injection attack on the connection path for CIFAR-10 (VGG). 
The error rate of clean/targeted samples means standard-test-error/attack-failure-rate, respectively.} 
\label{fig: injection_cifar_VGG}
%\vspace{-2mm}
\end{figure}

\begin{table}[tb]
\begin{center}
\caption{Performance against error-injection attack. The clean/injection accuracy means standard-test-accuracy/attack-success-rate, respectively.
Path connection has the best clean accuracy and can
completely remove injected errors  (i.e. 0\% attack accuracy).  More results are given in  Appendix \ref{appen_injection_more}.}
\label{tab: injection_comparison}
\scalebox{0.8}{
\begin{threeparttable}
\begin{tabular}{c|c|c|c|c|c|c|c|c|c|c|c}
\toprule[1pt]
&  &  \multicolumn{5}{c|}{Clean Accuracy} &  \multicolumn{5}{c}{Injection Accuracy} \\
\hline
Dataset & Method / Bonafide data size & 2500 & 1000  & 500  &  250 & 50 & 2500 & 1000  & 500  &  250 & 50\\
\hline
\multirow{6}{*}{\makecell{CIFAR-10 \\ (VGG)}}   & Path connection $(t=0.1)$ &  92\% & 90\% & 90\% & 90\% & 88\%  &0\% &0\%&0\%&0\%&0\% \\
& Fine-tune &  88\% & 87\% & 86\% & 84\% & 82\% &0\%&0\%&0\%&0\%&0\%\\
& Train from scratch &  45\% & 37\% & 27\% & 25\% & 10\% &0\%&0\%&0\%&0\%&0\%\\
& Noisy model ($t=0$) & 14\% & 14\% & 14\% & 14\% & 14\% & 36\%& 36\%& 36\%& 36\%& 36\%\\
&  Noisy model ($t=1$) &  12\% &  12\%&  12\%&  12\%&  12\% &19\%&19\%&19\%&19\%&19\%\\
&  Prune  & 91\% & 89\% & 88\%  & 88\% & 88\% &0\%&0\%&25\%&25\%&25\% \\
\hline
\multirow{6}{*}{\makecell{SVHN \\ (ResNet)}}   & Path connection $(t=0.1)$ & 96\%  & 94\%  & 92\% & 91\%  & 90\%  &0\% &0\% &0\% &0\% &0\% \\
& Fine-tune &  94\% & 93\% & 91\% & 89\%  & 88\% &0\% &25\% &0\% &25\%&25\% \\
& Train from scratch &  90\%  &  83\%  & 75\% &  61\%  & 21\%  &0\% &0\% &0\% &0\% &0\% \\
& Noisy model ($t=0$) &  11\%  &  11\% &  11\% &  11\% &  11\% &28\% &28\% &28\% &28\% &28\% \\
&  Noisy model ($t=1$) &  11\%  &  11\% &  11\% &  11\% &  11\% &18\%&18\%&18\%&18\%&18\%\\
&  Prune & 95\% & 93\% & 92\% &  90\%  & 89\% & 0\%& 0\%& 25\%& 0\%& 25\% \\
\bottomrule[1pt]
\end{tabular}
\end{threeparttable}
}
\end{center}
\vspace{-4mm}
\end{table}

\textbf{Evaluation and analysis.}
Similar to the setup in Section \ref{subsec_backdoor}, Figure \ref{fig: injection_cifar_VGG} shows the clean accuracy and attack accuracy of the models on the path connecting two error-injected models using limited amount of bonafide data. For the error-injected models ($t=\{0,1\}$), the attack accuracy is nearly 100\%, which corresponds to 0\% attack failure rate on targeted samples. However, using path connection and limited amount of bonafide data, the injected errors can be removed almost completely. Varying the size of path training data consistently sanitizes the error-injected models and
mainly affects the standard test error.
Most of the models on the path can attain nearly 100\% fault tolerance (i.e. 100\% attack failure rate) to the injected errors. The models on the path near $t=0$ or $t=1$ have comparable performance on clean data and exhibit strong fault tolerance to injected errors.
Similar findings are observed across different network architectures and datasets (see Appendix \ref{appen_injection_more}).

%For the path connection, we set  the start and end models are two models which   have already been modified by the fault injection attacker to achieve the misclassification of given images.  With the start and end models, we try to find a path to connect them with only a part of the test set images. Note that we only use limited number of images instead of the whole training or test set  to train the connection. This is consistent with the assumption that the users or defenders may not have large amount of clean images. Thus it is more  desirable to achieve defense with small number of images. Note that the images used to train the path connection and the images used to test the model's accuracy on the path do not overlap. 

%As shown in Figure XX and XX, on the start and end points with $t=0$ or $1$, the successful fault injection attack can achieve a very low error rate to classify the certain given images to the target labels. However, on the connection path except the start or end point, the error rate of the given images for the target labels are kept very high (almost 100\%) demonstrating that path connection can significantly remove the effects of the injected faults. We also notice that the less data used to train the connection, the lower accuracy on the connection path. It is reasonable as using small number of images for training will normally degrade the generalization performance.  

\textbf{Comparison with baselines and extensions.}
 In Table \ref{tab: injection_comparison}, we adopt the same baselines as in Section \ref{subsec_backdoor} to compare with path connection. We find that only path connection and training-from-scratch can successfully sanitize the error-injected models and attain
0\% attack accuracy, and other baselines are less effective.
Table \ref{tab: injection_comparison} also shows the clean accuracy of path connection is substantially better than the effective baselines, suggesting novel applications of mode connectivity for finding accurate and adversarially robust models. The extensions to other settings are discussed in Appendix \ref{appen_one_model}. 

\textcolor{black}{
\textbf{Technical explanations and adaptive attack.} Consistent with the results in backdoor attacks, we explore the model weight space to demonstrate the significant difference between the models on our connection path and random noisy models. We also show that the similarity of error-injected images are much lower than that of clean images. In addition, our path connection is resilient to the advanced path-aware error-injection attack. More details are given in Appendices \ref{appen_noisy_model}, \ref{appen_similarity} and \ref{appen_adaptive}.
}

%In this section, we compare the performance of path connection against injection attack with two baseline methods, fine-tuning and random start. The fine-tuning method would start from an injected model and we use limited number of test set images to train this model. By contrast, the random start method would train a model with random initialization. 
%We use different number of images (2500, 1000, 500, 250, and 50 images) to train the connection or the model with/without random start and summarize their performance in Table \ref{tab: injection_comparison}. We adopt the VGG16 architecture on CIFAR-10 dataset.
%Note that for the path connection method, we train the connection for 100 epochs and then choose the model at $t=0.1$ on the path.  For the fine-tune and random start methods, we choose the model after training for 100 epochs. 
%All of these methods can achieve zero accuracy for the given injected errors, which means the effects of injected errors have been eliminated from the current models. However, clean test set accuracy of the three methods are different as demonstrated in Table \ref{tab: injection_comparison}. 

\subsection{Robustness Loss Landscape and 
Correlation Analysis for Evasion Attack}
To gain insights on mode connectivity against evasion attacks, here we investigate the standard and adversarial-robustness loss landscapes on the same path connecting two untampered and independently trained models. The path is trained using the entire training dataset for minimizing \eqref{eqn_path_loss} with standard cross entropy loss. The robustness loss refers to the cross entropy of class predictions on adversarial examples generated by evasion attacks and their original class labels. Higher robustness loss suggests the model is more vulnerable to evasion attacks. In addition, we will investigate the robustness loss landscape connecting regular (non-robust) and adversarially-trained (robust) models, where the path is also trained with standard cross entropy loss. We will also study the behavior of the largest eigenvalue of the Hessian matrix associated with the cross entropy loss and the data input, which we call the input Hessian. As adversarial examples are often generated by using the input gradients, we believe the largest eigenvalue of the input Hessian can offer new insights on robustness loss landscape, similar to the role of model-weight Hessian on quantifying generalization performance \citep{wu2017towards,wang2018identifying}.

\begin{figure}[t]    
\hspace{-2mm}
 \centering
\begin{tabular}{p{0.1in}p{1.25in}p{1.25in}p{1.25in}p{0.85in}}
 & \parbox{1.25in}{\centering \footnotesize Connection of regular models (PCC=0.71)} &  
\parbox{1.25in}{\centering \footnotesize  Connection of regular and adversarially-trained models (PCC=0.88)}  &  
\parbox{1.25in}{\centering \footnotesize Connection of adversarially-trained models (PCC=0.86) }  & 
\parbox{0.85in}{\centering \footnotesize Legend  } \\
\vspace{-0.4in} \rotatebox{90}{\parbox{0.6in}{\centering \footnotesize loss \& eigenvalue }}  &  \includegraphics[align=c,width=1.25in]{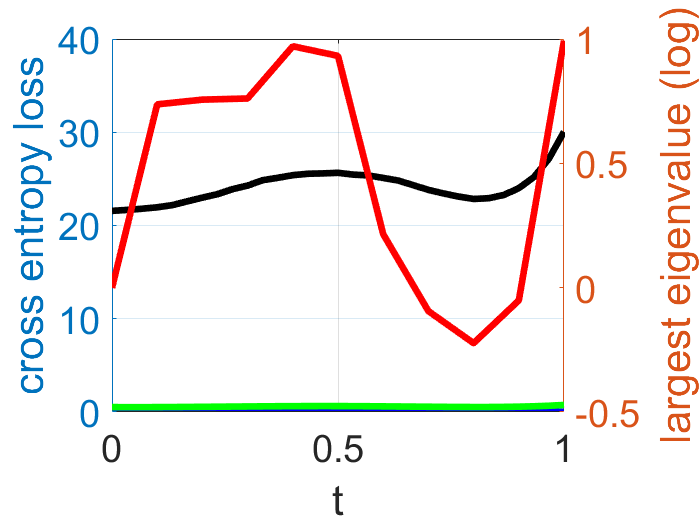} &
\includegraphics[align=c,width=1.25in]{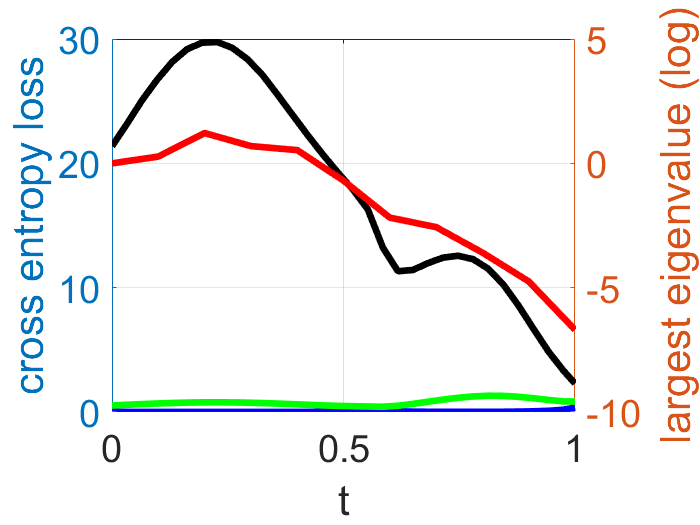} & 
\includegraphics[align=c,width=1.25in]{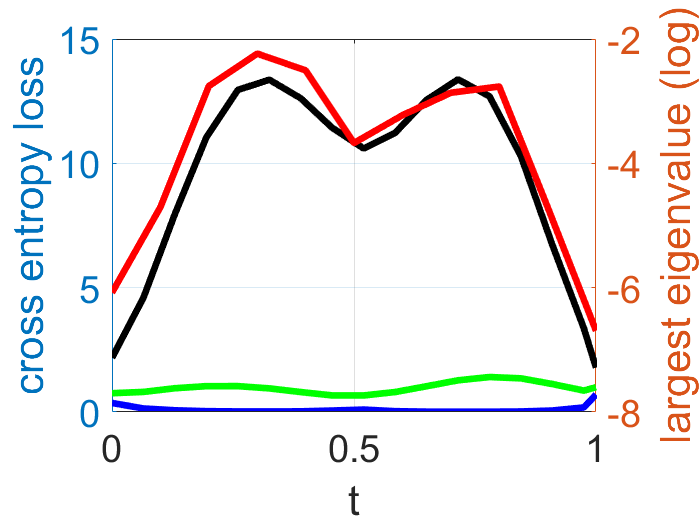} &\includegraphics[align=c,width=0.9in]{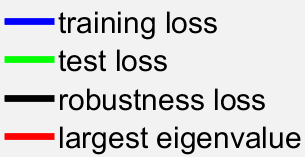} 
\\ 
\vspace{-0.5in} \rotatebox{90}{\parbox{0.9in}{\centering \footnotesize error rate \& \\ attack success rate }} & \includegraphics[align=c,width=1.25in]{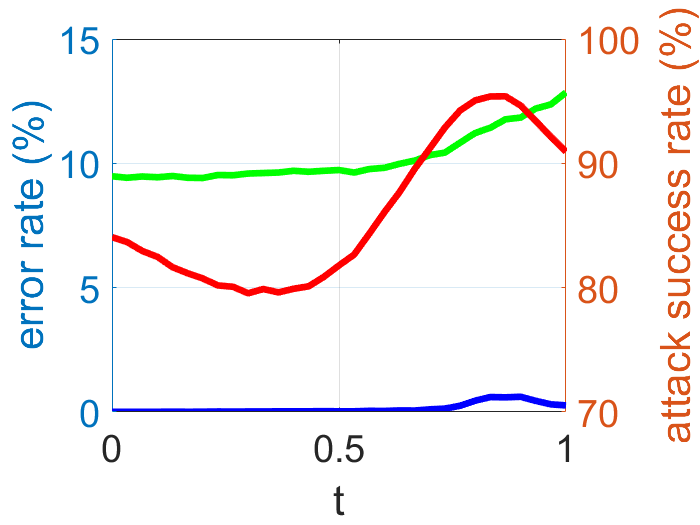} &
\includegraphics[align=c,width=1.25in]{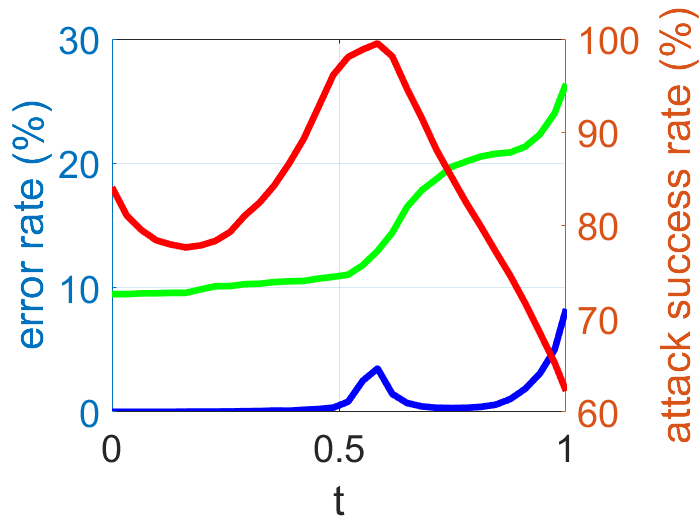} & 
\includegraphics[align=c,width=1.25in]{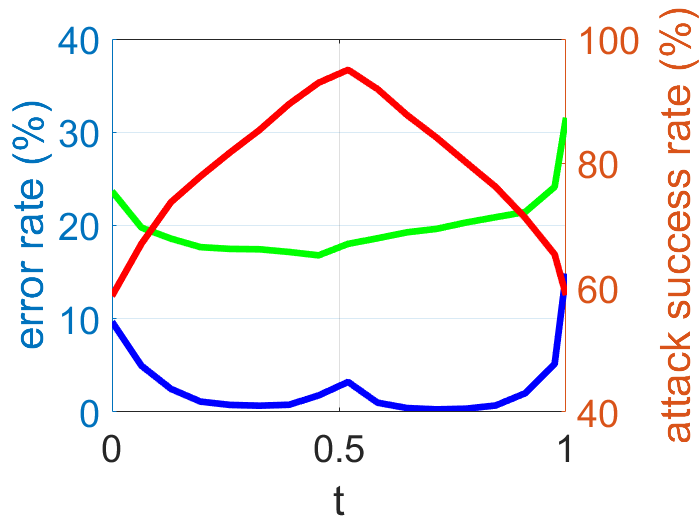}
&\includegraphics[align=c,width=0.9in]{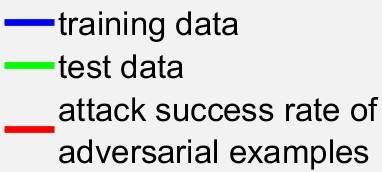}
\end{tabular}
\vspace{-2mm}
\caption{Loss, error rate, attack success rate and largest eigenvalue of input Hessian on the path connecting different model pairs on CIFAR-10 (VGG) using standard loss. The error rate of training/test data means standard training/test error, respectively. In all cases, there is no standard loss barrier but a robustness loss barrier.
There is also a high correlation between the robustness loss and the largest eigenvalue of input Hessian, and their Pearson correlation coefficient (PCC) is reported in the title.} 
\label{fig: input_hessian_vs_adversarial_loss}
\vspace{-4mm}
\end{figure}

\textbf{Attack Implementation.}
We uniformly select 9 models $(t=\{0.1,0.2,\ldots,0.9\})$ on the path and run evasion attacks on each of them using the $\ell_\infty$-norm-ball based projected gradient descent (PGD) method proposed in \citep{madry2017towards}.  The robustness loss is evaluated using the non-targeted adversarial examples crafted from the entire test set, and the attack perturbation strength is set to $\epsilon=8/255$ with 10 iterations.
We also use the PGD method for adversarial training to obtain adversarially-trained models that are robust to evasion attacks but pay the price of reduced accuracy on clean data \citep{madry2017towards}.

\textbf{Evaluation and Analysis.} To study the standard and robustness loss landscapes, we scrutinize the models on the path connecting the following pairs of models: (i) independently trained regular (non-robust) models; (ii) regular to adversarially-trained models; and (iii) independently adversarially-trained models. These results are shown in Figure \ref{fig: input_hessian_vs_adversarial_loss}. We summarize the major findings as follows.
\begin{itemize}[leftmargin=*]
\item \textsf{No standard loss barrier in all cases:} Regardless of the model pairs, all models on the paths have similar standard loss metrics in terms of training and test losses, which are consistent with the previous results on the ``flat'' standard loss landscape for mode connectivity \citep{garipov2018loss,gotmare2018using,draxler2018essentially}. The curve of standard loss in case (ii) is skewed toward one end due to the artifact that the adversarially-trained model ($t=1$) has a higher training/test error than the regular model ($t=0$). 

\item \textsf{Robustness loss barrier:} Unlike standard loss, we find that the robustness loss on the connection path has a very distinct characteristic. In all cases, there is a robustness loss barrier (a hill) between pairs of regular and adversarially-trained models. The gap (height) of the robustness loss barrier is more apparent in cases (ii) and (iii).
For (ii), the existence of a barrier suggests the modes of regular and adversarially-trained models are not connected by the path in terms of robustness loss, which also provides a geometrical evidence of the ``no free lunch'' hypothesis that adversarially robust models cannot be obtained without additional costs \citep{tsipras2019robustness,dohmatob2018limitations,bubeck2018adversarial}. For (iii), robustness loss barriers also exist. The models on the path are less robust than the two adversarially-trained models at the path end, despite they have similar standard losses. The results suggest that there are essentially no better adversarially robust models on the path connected by  regular training using standard loss.

\item \textsf{High correlation between the largest eigenvalue of input Hessian and robustness loss:} Inspecting the largest eigenvalue of input Hessian $H_t(x)$ of a data input $x$ on the path, denoted by $\lambda_{\max}(t)$,  we observe a strong accordance between $\lambda_{\max}(t)$ and robustness loss on the path, verified by the high empirical Pearson correlation coefficient (PCC) averaged over the entire test set as reported in Figure \ref{fig: input_hessian_vs_adversarial_loss}. As evasion attacks often use input gradients to craft adversarial perturbations to $x$, the eigenspectrum of input Hessian  indicates its local loss curvature and relates to adversarial robustness \citep{yu2018interpreting}. The details of computing $\lambda_{\max}(t)$ are given in Appendix \ref{appen_evasion}. Below we provide technical explanations for the empirically observed high correlation between $\lambda_{\max}(t)$ and the oracle robustness loss on the path, defined as $\max_{\|\delta\|\leq \epsilon}  l(w(t),x+\delta)$.
\begin{proposition}
\label{prop_correlation}
Let $f_{w}(\cdot)$ be a neural network classifier with its model weights denoted by $w$ and let $l(w,x)$ denote the classification loss (e.g. cross entropy of $f_w(x)$ and the true label $y$ of a data sample $x$). Consider the oracle robustness loss  $\max_{\|\delta\| \leq \epsilon}\ell(w(t),x+\delta)$ of the model $t$ on the path, where $\delta$ denotes a perturbation to $x$ confined by an $\epsilon$-ball induced by a vector norm $\|\cdot\|$. 
Assume \\
     (a) the standard loss $l(w(t),x)$ on the path is a constant for all $t \in [0,1]$. \\
     (b) $l(w(t),x+\delta) \approx l(w(t),x) + \nabla_x l(w(t),x)^T \delta + \frac{1}{2} \delta^T H_t(x) \delta$ for small $\delta$, where $\nabla_x l(w(t),x)$ is the input gradient and $H_t(x)$ is the input Hessian of $l(w(t),x)$ at $x$. \\
%     (c) the largest eigenvector $v$ of $H_t(x)$ and $\nabla_x l(w(t),x)$ satisfy $\frac{|\nabla_x l(w(t),x)^T v|}{\|\nabla_x l(w(t),x)\|} \geq c$. \\
     Let $c$ denote the normalized inner product in absolute value for the largest eigenvector $v$ of $H_t(x)$ and $\nabla_x l(w(t),x)$, $\frac{|\nabla_x l(w(t),x)^T v|}{\|\nabla_x l(w(t),x)\|} = c$. 
Then we have $\max_{\|\delta\|\leq \epsilon}  l(w(t),x+\delta) \sim \lambda_{\max}(t)$ as $c \rightarrow 1$.
\end{proposition}
\textbf{Proof:}
The proof is given in Appendix \ref{appen_prop_correlation}.
Assumption (a) follows by the existence of high-accuracy path of standard loss landscape from mode connectivity analysis. Assumption (b) assumes the local landscape with respect to the input $x$ can be well captured by its
second-order curvature based on Taylor expansion. The value of $c$ is usually quite large, which has been empirically verified in both regular and adversarially-trained models \citep{moosavi2018robustness}.
\end{itemize}

\textbf{Extensions.} Although we find that there is a robustness loss barrier on the path connected by regular training, we conduct additional experiments to show that it is possible to find an robust path connecting two adversarially-trained or regularly-trained model pairs using adversarial training \citep{madry2017towards}, which we call the ``robust connection'' method.  However, model ensembling using either the regular connection or robust connection has little gain against evasion attacks, as the adversarial examples are known to transfer between similar models \citep{papernot2016transferability,su2018robustness}. We refer readers to Appendix \ref{appen_robust_conn} for more details.

\section{Conclusion}
This paper provides novel insights on adversarial robustness of deep neural networks through the lens of mode connectivity in loss landscapes. Leveraging mode connectivity between model optima, we show that path connection trained by a limited number of clean data can successfully repair backdoored or error-injected models and significantly outperforms several baseline methods. Moreover, we use mode connectivity to uncover the existence of robustness loss barrier on the path trained by standard loss against evasion attacks. We also provide technical explanations for the effectiveness of our proposed approach and theoretically justify
the empirically observed high correlation between robustness loss and the largest eigenvalue of input Hessian. Our findings are consistent and validated on different network architectures and datasets.

\section*{Acknowledgements}
This work was primarily conducted during Pu Zhao's internship at IBM Research.
This work is partly supported by the National Science Foundation CNS-1929300.

\bibliography{adversarial_learning}
\bibliographystyle{iclr2020_conference}

\clearpage
\appendix

\section*{Appendix}
\setcounter{equation}{0}
\setcounter{figure}{0}
\setcounter{table}{0}
\makeatletter
\renewcommand{\theequation}{A\arabic{equation}}
\renewcommand{\thefigure}{A\arabic{figure}}
\renewcommand{\thetable}{A\arabic{table}}
%\renewcommand{\bibnumfmt}[1]{[S#1]}
%\renewcommand{\citeppnumfont}[1]{S#1}
%%%%%%%%%% Prefix a "S" to all equations, figures, tables and reset the counter %%%%%%%%%%
\appendix

\section{Network Architecture and Training}
\label{appen_network}
In this paper we mainly use two model architectures, VGG and ResNet. The VGG model \citep{simonyan2014very} has 13 convolutional layers and 3 fully connected layers.  The ResNet model is based on the Preactivation-ResNet implementation \citep{he2016identity} with 26 layers.
The clean test accuracy of untampered models of different architectures on CIFAR-10 and SVHN are given in Table \ref{tab: clean_accuracy_models_datasets_backdoor}. 
All of the experiments are performed on 6 GTX 1080Ti GPUs. The experiments are implemented with Python and Pytorch. 

%\vspace{-2mm}
\begin{table}[h]
\begin{center}
\caption{Test accuracy of untampered models for different datasets and model architectures.}
\label{tab: clean_accuracy_models_datasets_backdoor}
\scalebox{1.0}{
\begin{threeparttable}
\begin{tabular}{c|c|c|c}
\toprule[1pt]
 & architecture &  VGG   & ResNet \\
\hline
\multirow{2}{*}{\makecell{CIFAR-10}}   & model $(t=0)$ &  88\% & 87\% \\
& model $(t=1)$ & 86\%  & 91\% \\
\hline
\multirow{2}{*}{\makecell{SVHN}}   & model $(t=0)$& 96\%  & 97\% \\
& model $(t=1)$ &  98\% &  99\%\\
\bottomrule[1pt]
\end{tabular}
\end{threeparttable}
}
\end{center}
\vspace{-4mm}
\end{table}

%\begin{table}[tb]
%\begin{center}
%\caption{Clean accuracy of models for different datasets and model architectures in %injection attack}
%\label{tab: clean_accuracy_models_datasets_injection}
%\scalebox{1.0}{
%\begin{threeparttable}
%\begin{tabular}{c|c|c|c}
%\toprule[1pt]
% & architecture &  VGG   & ResNet \\
%\hline
%\multirow{2}{*}{\makecell{CIFAR-10}}   & model $(t=0)$ & 78\%  & 82\%\\
%& model $(t=1)$ &  75\%  &  84\% \\
%\hline
%\multirow{2}{*}{\makecell{SVHN}}   & model $(t=0)$&  86\%  & 92\%\\
%& model $(t=1)$ & 85\%  & 91\% \\
%\bottomrule[1pt]
%\end{tabular}
%\end{threeparttable}
%}
%\end{center}
%\end{table}

\section{Regular Path Connection of Untampered Models on SVHN (ResNet)}
\label{appen_untampered}
The performance of regular path connection of untampered models on SVHN with ResNet is presented in Figure \ref{fig: regular_svhn_res}.

\begin{figure}[h]    
\hspace{-5.5mm}
 \centering
\begin{tabular}{p{1.9in}p{1.9in}p{1.2in}}
\parbox{1.9in}{\centering Inference on training set}
 & \parbox{1.9in}{\centering  Inference on test set}
  & \parbox{1.2in}{\centering  Legend}\\
\includegraphics[align=c,width=1.9in]{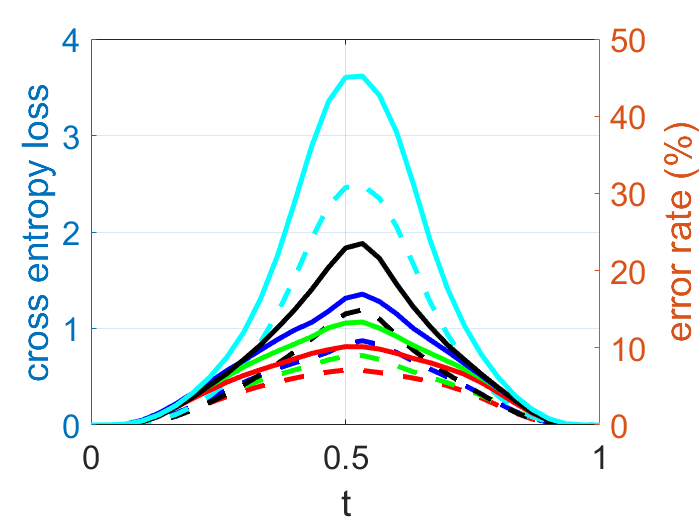}  &
\includegraphics[align=c,width=1.9in]{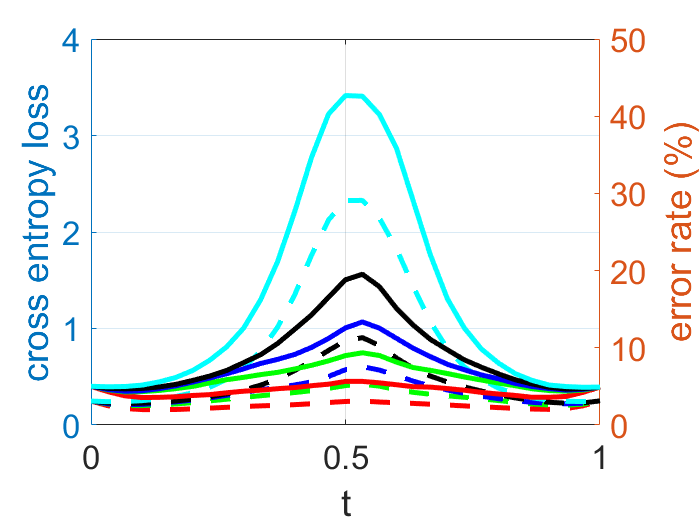} & 
\includegraphics[align=c,width=1.35in]{figs_new/legend2.png} 
\end{tabular}
\vspace{-2mm}
\caption{Loss and error rate on the path connecting two untampered ResNet models trained on SVHN. The path connection is trained using different settings as indicated by the curve colors. The inference results on test set  are evaluated  using \textcolor{black}{5000} samples, which are separate from what are used for path connection.  }
\label{fig: regular_svhn_res}
\vspace{-4mm}
\end{figure}

\section{Illustration and Implementation Details of Backdoor and Error-Injection Attacks}
\label{appen_backdoor}

\paragraph{Backdoor attack}
The backdoor attack is implemented by poisoning the training dataset and then training a backdoored model with this training set. To poison the training set, we randomly pick 10\% images from the training dataset and add a trigger to each image. The shape and location of the trigger  is shown in Figure \ref{fig: example}. Meanwhile, we set the labels of the triggered images to the target label(s) as described in Section \ref{subsec_backdoor}. 

\begin{figure}[h]    
 \centering
\begin{tabular}{p{0.1in}p{1.1in}p{1.1in}p{1.1in}p{1.1in}}
 \vspace{-0.35in} \rotatebox{90}{\parbox{0.6in}{\centering \footnotesize CIFAR-10 }}  &  \includegraphics[align=c,width=1.1in]{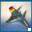} &
\includegraphics[align=c,width=1.1in]{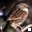} & 
\includegraphics[align=c,width=1.1in]{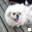} &   \includegraphics[align=c,width=1.1in]{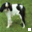}
\\ 
\\
 \vspace{-0.35in} \rotatebox{90}{\parbox{0.6in}{\centering \footnotesize SVHN }} & \includegraphics[align=c,width=1.1in]{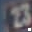} &
\includegraphics[align=c,width=1.1in]{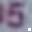} & 
\includegraphics[align=c,width=1.1in]{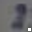}   &  \includegraphics[align=c,width=1.1in]{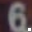}
\end{tabular}
\caption{Examples of backdoored images on CIFAR-10 and SVHN. The triggers are white blocks located at the right-bottom area of each image.} \vspace{-2mm}
\label{fig: example}
\end{figure}

\paragraph{Error-injection attack}

We select 1000 images from the test set and pick 4 images as targeted samples with randomly selected  target labels for inducing attack. The target labels of the 4 selected images are different from their original correct labels. The goal of the attacker is to change the classification of the 4 images to the target labels while keeping the classification of the remaining 996 images unchanged through modifying the model parameters.
To obtain the models with injected errors on CIFAR-10, we first train two models with a clean accuracy of 88\% and 86\%, respectively. %The attacker's goal is to  inject 5 errors or change the classification of these 4 targeted images while keeping the classification of the left 996 images unchanged by modifying the last fully connected layer in the model.
Keeping the classification of a number of images unchanged can help to mitigate the accuracy degradation incurred by the model weight perturbation.  After perturbing the model weights, the 4 errors can be injected into the model successfully with 100\% accuracy for their target labels. The accuracy for other clean images become 78\% and 75\%, respectively.

\section{Prediction Error on the Triggered Data}
\label{appen_true_error}
We show the prediction error of the triggered data relative to the true labels on all datasets and networks in Figure \ref{fig: backdoor_cifar_VGG2}. The error rate means the fraction of triggered images having top-1 predictions different from the original true labels. The prediction error rate of triggered data is high at path ends ($t=0,1$) since the two end models are tampered. It has similar trend as standard test error for models not too close to the path ends, suggesting path connection can find models having good classification accuracy on triggered data. 

\begin{figure}[h]   
 \centering
\begin{tabular}{p{1.2in}p{1.2in}p{1.2in}p{1.2in}}
 \parbox{1.2in}{\footnotesize \centering CIFAR-10 (VGG)} & 
 \parbox{1.2in}{\footnotesize \centering CIFAR-10 (ResNet)} & 
 \parbox{1.2in}{\footnotesize \centering SVHN (VGG)} & 
 \parbox{1.2in}{\footnotesize \centering SVHN (ResNet)} 
 \\
\includegraphics[align=c,width=1.2in]{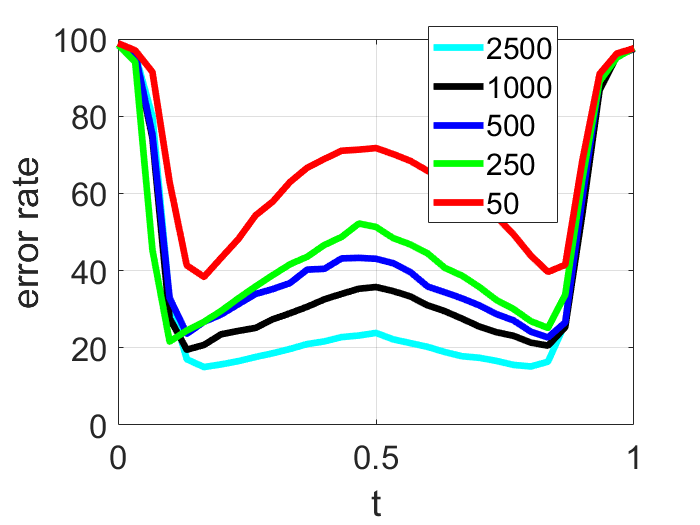}  &
\includegraphics[align=c,width=1.2in]{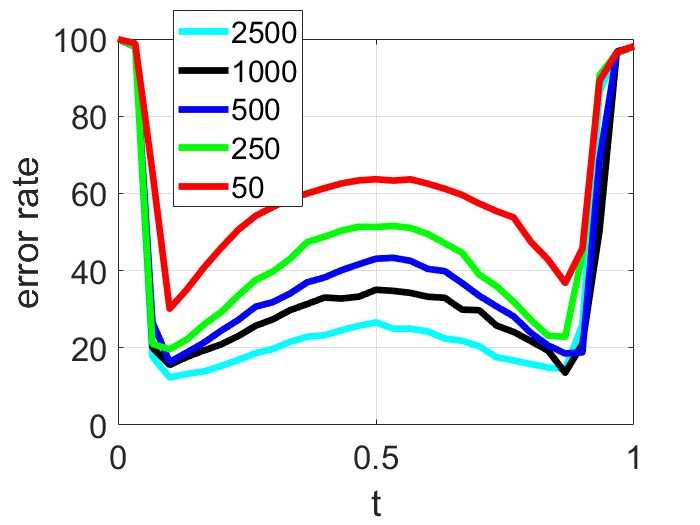} & 
\includegraphics[align=c,width=1.2in]{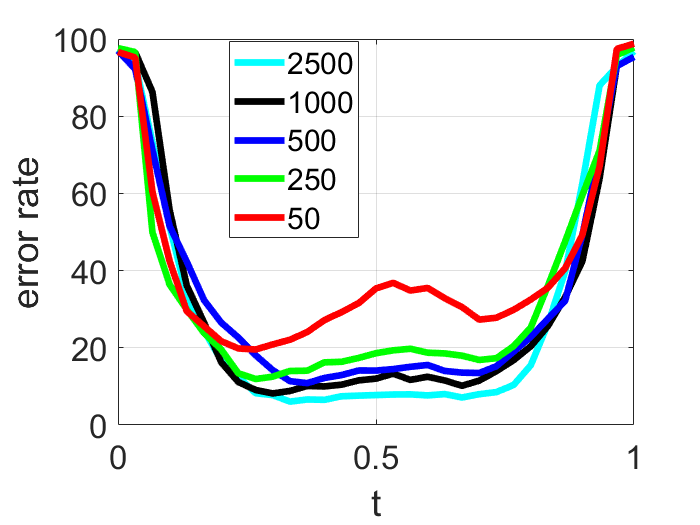} &
\includegraphics[align=c,width=1.2in]{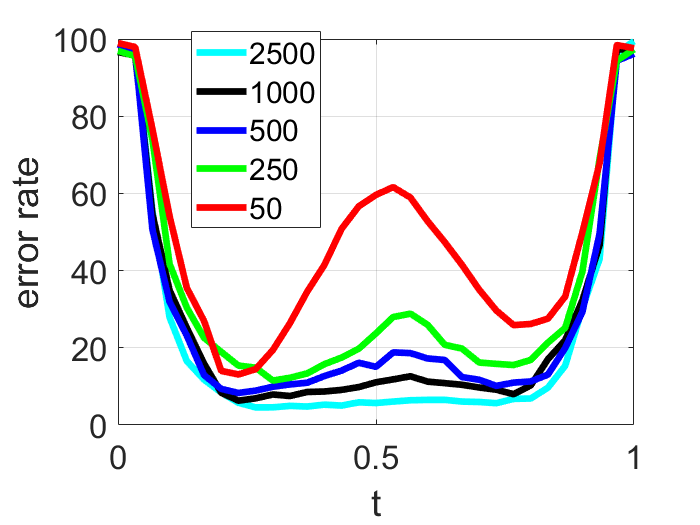} 
\\
\end{tabular}
\caption{Prediction error rate against  backdoor attacks on the connection path.}
\label{fig: backdoor_cifar_VGG2}
\end{figure}

\section{More Results on Path Connection against Backdoor Attacks}
\label{appen_backdoor_more}

\paragraph{Implementation Details for 
\label{appen_baseline}
Table \ref{tab: backdoor_comparison}}
For our proposed path connection method, we train the connection using different number of images as given in Table \ref{tab: backdoor_comparison} for 100 epochs and then report the performance of the model associated with a selected index $t$ on the path.  For the fine-tuning and training-from-scratch methods, we report the model performance after training for 100 epochs. For the random Gaussian perturbation to model weights, we evaluate the model performance under Gaussian noise perturbations on the model parameters. There are two given models which are the models at $t=0$ and $t=1$. The Gaussian noise has zero mean with a standard deviation of  the absolute value of the difference between the two given models. Then we add the Gaussian noise to the two given models respectively and test their accuracy for clean and triggered images. For Gaussian noise, the experiment is performed multiple times (50 times) and we report the average accuracy.   
We can see that adding Gaussian noise perturbations to the model does not necessarily change the model status from robust to non-robust or from non-robust to robust. The path connection or evolution  from the model at    $t=0$ to that $t=1$ follows a specific path achieving robustness against backdoor attack rather than random exploration. For pruning, we use the filter pruning method \citep{li2016pruning} to prune filters from convolutional  neural networks (CNNs) that are identified as having a small effect on the output accuracy.   By removing the whole filters in the network together with their connecting feature maps, the computation costs are reduced significantly.  We first prune about 60\% of its parameters for VGG or 20\% parameters for ResNet. Then we retrain the network with different number of images  as given in Table \ref{tab: backdoor_comparison} for 100 epochs. The clean accuracy and backdoor accuracy are as reported.

Figure \ref{fig: backdoor_cifar_appendix} shows the error rates of clean and backdoored samples using CIFAR-10 on the connection path  against single-target attacks.

\begin{figure}[h]    
 \centering
\begin{tabular}{p{1.9in}p{1.9in}p{1.2in}}
\parbox{1.9in}{\centering VGG}
 & \parbox{1.9in}{\centering ResNet}
   & \parbox{1.2in}{\centering Legend}\\
\includegraphics[align=c,width=1.9in]{figs_new/single_target_backdoor_cifar_vgg.png} 
&
\includegraphics[align=c,width=1.9in]{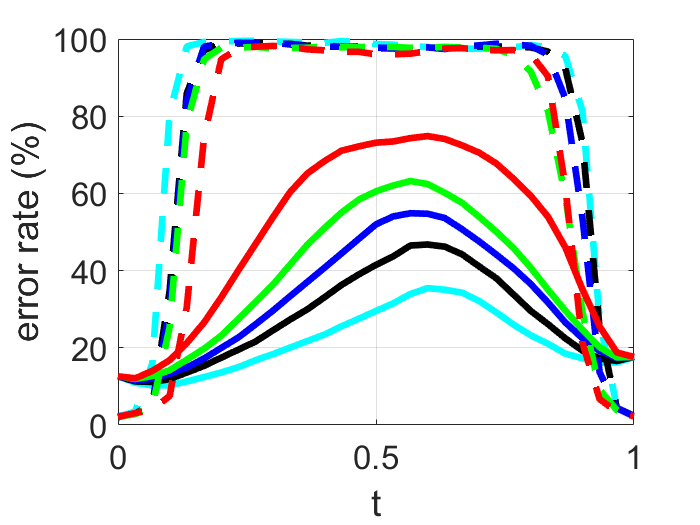} &
\includegraphics[align=c,width=1.2in]{figs_new/legend.png} 
\end{tabular}
\caption{Error rate of single-target backdoor attack on the  connection path for CIFAR-10. The error rate of clean/backdoored samples means standard-test-error/attack-failure-rate, respectively.} 
\label{fig: backdoor_cifar_appendix}
\end{figure}

Figure \ref{fig: backdoor_svhn} shows the error rates of clean and backdoored samples using SVHN on the connection path  against single-target attacks.

\begin{figure}[h]    
 \centering
\begin{tabular}{p{1.9in}p{1.9in}p{1.2in}}
\parbox{1.9in}{\centering VGG}
 & \parbox{1.9in}{\centering  ResNet}
   & \parbox{1.2in}{\centering  Legend}
   \\
\includegraphics[align=c,width=1.9in]{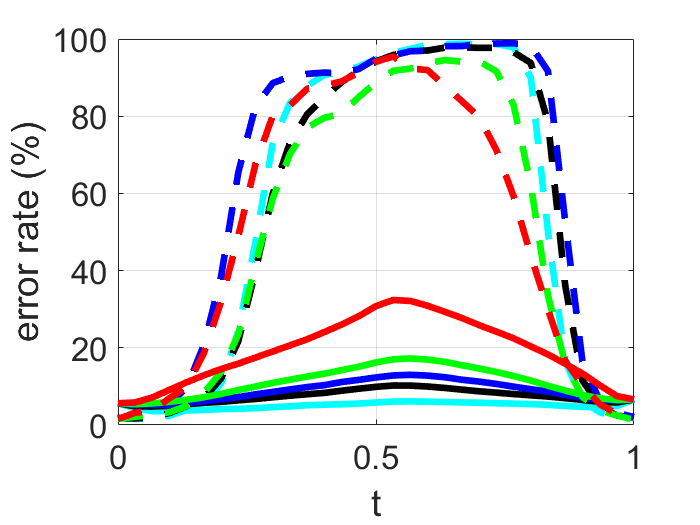} 
&
\includegraphics[align=c,width=1.9in]{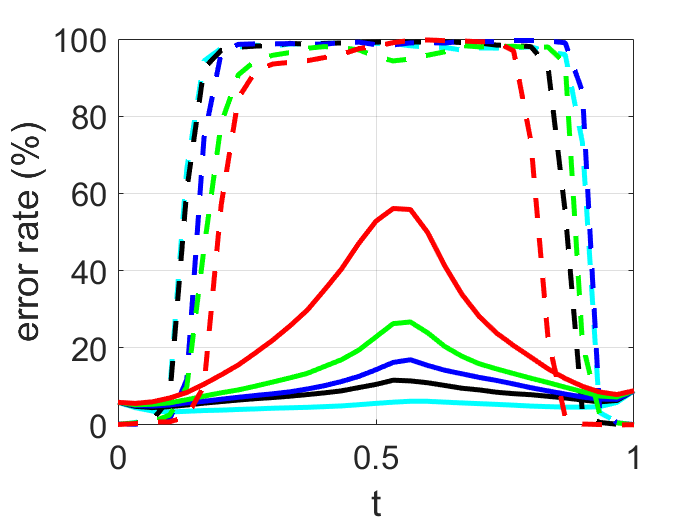} &
\includegraphics[align=c,width=1.2in]{figs_new/legend.png} 
\end{tabular}
\caption{Error rate of single-target backdoor attack on the connection path for SVHN. The error rate of clean/backdoored samples means standard-test-error/attack-failure-rate, respectively.} 
\label{fig: backdoor_svhn}
\end{figure}

Table \ref{tab: backdoor_comparison_2} shows the performance comparison of path connection and other baseline methods 
against single-target backdoor attack evaluated on CIFAR-10 (ResNet) and SVHN (VGG).

\begin{table}[h]
\begin{center}
\caption{Performance comparison of path connection and baselines against single-target backdoor attack.  The clean/backdoor accuracy means standard-test-accuracy/attack-success-rate, respectively.} 
\label{tab: backdoor_comparison_2}
\scalebox{0.9}{
\begin{threeparttable}
\begin{tabular}{c|c|c|c|c|c|c|c}
\toprule[1pt]
& & Methods / bonafide data size & 2500 & 1000  & 500  &  250 & 50\\
\midrule[1pt]
\multirow{ 12}{*}{\makecell{CIFAR-10 \\ (ResNet)}}   &\multirow{ 6}{*}{ \makecell{Clean \\  Accuracy}}&Path connection  ($t=0.2$) & 87\% &  83\% & 80\%  & 77\%  &  67\% \\
& &Fine-tune &  85\% &  84\% &  83\% & 83\% & 72\% \\
& &Train from scratch & 42\%  & 36\%  & 33\% &  29\% &  22\% \\
& & Noisy model ($t=0$) & 12\% & 12\% & 12\%& 12\%& 12\% \\
& & Noisy model ($t=1$) &  10\%  &  10\%&  10\%&  10\%&  10\%\\
& &  Prune  &   85\%  & 82\% & 80\%& 79\% & 77\%  \\
\cline{2-8}
&\multirow{ 6}{*}{ \makecell{Backdoor \\ Accuracy }}&Path connection  ($t=0.2$) &  0.6\% & 1.2\% & 1.3\% & 2.3\% & 5.3\% \\
&  & Fine-tune &  0.7\%  & 8.7\% & 19\% & 20\% & 18\%\\
&  & Train from scratch & 3.7\%  & 4.3\% & 3.5\% &  8.2\% &  6.5\%  \\
& &  Noisy model ($t=0$) & 95\%  & 95\%  & 95\%  & 95\%  & 95\% \\
& &  Noisy model ($t=1$) & 97\% & 97\%& 97\%& 97\%& 97\% \\
& &  Prune  &   37\%& 52\% & 78\% & 86\% &88\%  \\
\midrule[1pt]
\multirow{12}{*}{\makecell{SVHN \\ (VGG)}}   &\multirow{ 6}{*}{ \makecell{Clean \\  Accuracy} }&Path connection   ($t=0.7$) & 95\% & 93\% & 91\% & 87\% & 75\% \\
& &Fine-tune & 93\%  &  92\%  & 90\%  & 89\% &  77\%\\
& &Train from scratch & 88\% & 82\% &  76\%  & 54\%  &  25\% \\
& & Noisy model ($t=0$) &   22\%  &   22\%  &   22\%  &   22\%  &   22\%  \\
& & Noisy model ($t=1$) &   16\% &   16\% &   16\% &   16\% &   16\%  \\
& &  Prune  &   94\% & 92\% & 91\% & 89\% & 86\% \\
\cline{2-8}
&\multirow{ 6}{*}{ \makecell{Backdoor \\ Accuracy }}&Path connection ($t=0.7$)& 2\% & 6\% & 2\% & 9\% &  22\% \\
&  & Fine-tune &   11\% &  10\% &  13\%  & 30\% &  61\%  \\
&  & Train from scratch & 2.6\% & 2.3\% & 3.3\% & 4.7\% &  13\% \\
& & Noisy model ($t=0$) & 79\%  & 79\%  & 79\%  & 79\%  & 79\% \\
& & Noisy model ($t=1$) &  62\% &  62\%&  62\%&  62\%&  62\%\\
& &  Prune  &   75\%  &69\%&  78\% &86\% & 91\%  \\
\bottomrule[1pt]
\end{tabular}
\end{threeparttable}
}
\end{center}
\end{table}

%The path connection can achieve the highest clean accuracy in different configurations of training image number. The reason may be that the model at $t=0$ or $1$ has a large influence for the models near it, so the models around them can achieve high clean test accuracy. 

\section{More Results on Path Connection against Error-Injection Attacks}
\label{appen_injection_more}

\paragraph{Implementation Details for Table \ref{tab: injection_comparison}}
\label{appen_baseline_injection}

For our proposed path connection method, we train the connection using different number of images as given in Table  \ref{tab: injection_comparison} for 100 epochs and then report the performance of the model associated with a selected index on the path. The start model and end model have been injected with the same  4  errors (misclassifying 4 given images), starting from two different unperturbed models obtained with different training hyper-parameters. 
For the fine-tuning and training-from-scratch methods, we report the model performance after training for 100 epochs.  For the random Gaussian perturbation to model weights, we evaluate the model performance under Gaussian noise perturbations on the model parameters.  There are two given models which are the models at $t=0$ and $t=1$.  The Gaussian noise has zero mean with a standard deviation of the absolute value of the difference between the two given models. Then we add the Gaussian noise to the two given models respectively and test their accuracy for clean and triggered images. The experiment is performed multiple times (50 times) and we report the average accuracy.  We can see that adding Gaussian noise perturbations to the model does not necessarily change the model status from robust to non-robust or from non-robust to robust. The path connection or evolution from the model at $t= 0$ to that $t= 1$ follows a specific path achieving robustness against backdoor attack rather than random exploration.
The training-from-scratch baseline method usually obtains the lowest clean test accuracy, especially in the case of training with 50 images, where training with so little images does not improve the accuracy.

Figure \ref{fig: injection_cifar_res} shows the performance of path connection against error-injection attacks evaluated on CIFAR-10 (ResNet) and SVHN (VGG).

\begin{figure}[h]    
 \centering
\begin{tabular}{p{1.9in}p{1.9in}p{1.2in}}
\parbox{1.9in}{\centering  CIFAR-10 (ResNet)}
 & \parbox{1.9in}{\centering  SVHN (VGG)}
  & \parbox{1.2in}{\centering  Legend} \\
\includegraphics[align=c,width=1.9in]{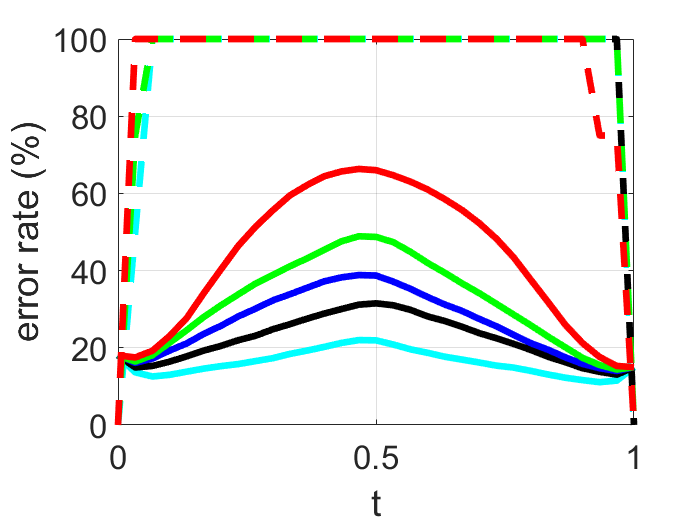}  &
\includegraphics[align=c,width=1.9in]{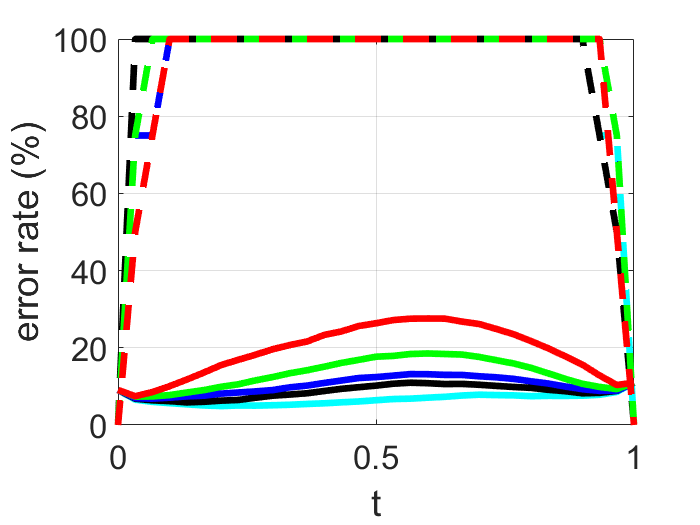} & 
\includegraphics[align=c,width=1.2in]{figs_new/legend.png} 
\end{tabular}
\caption{Error rate against error-injection attack on the connection path for CIFAR-10 (VGG). 
The error rate of clean/targeted samples means standard-test-error/attack-failure-rate, respectively.} 
\label{fig: injection_cifar_res}
\end{figure}

Table \ref{tab: injection_comparison} and \ref{tab: injection_comparison_full} 
show the performance comparison of path connection and other baseline methods against error-injection attacks evaluated on all combinations of network architectures and datasets.

\begin{table}[h]
\begin{center}
\caption{Performance evaluation against error-injection attack. The clean/injection accuracy means standard-test-accuracy/attack-success-rate, respectively. Path connection attains the highest clean accuracy and lowest (0\%) attack accuracy among all methods.}
\label{tab: injection_comparison_full}
\scalebox{0.92}{
\begin{threeparttable}
\begin{tabular}{c|c|c|c|c|c|c|c}
\toprule[1pt]
& & Methods / bonafide data size & 2500 & 1000  & 500  &  250 & 50\\
\midrule[1pt]
%\multirow{ 12}{*}{\makecell{CIFAR-10 \\ (VGG)}}   &\multirow{ 6}{*}{ \makecell{Clean \\  Accuracy}}&Path connection ($t=0.1$) & 92\% & 90\% & 90\% & 90\% & 88\% \\
%& & Fine-tune &  88\% & 87\% & 86\% & 84\% & 82\% \\
%& & Train from scratch &  45\% & 37\% & 27\% & 25\% & 10\% \\
%& & Noisy model ($t=0$) & 14\% & 14\% & 14\% & 14\% & 14\% \\
%& & Noisy model ($t=1$) &  12\% &  12\%&  12\%&  12\%&  12\% \\
%&  &Prune  & 91\% & 89\% & 88\%  & 88\% & 88\%  \\
%\cline{2-8}
%&\multirow{ 6}{*}{ \makecell{Injection \\ Accuracy }}&Path connection  ($t=0.1$)&  0\%  &  0\%&  0\%&  0\%&  0\%\\
%&  & Fine-tune & 0\% &  0\% &  0\%&  0\%&  0\%\\
%&  & Train from scratch & 0\% &  0\%&  0\%&  0\%&  0\% \\
%& &  Noisy model ($t=0$) & 36\%  & 36\%  & 36\%  & 36\%  & 36\%  \\
%& &  Noisy model ($t=1$) & 19\% &19\%& 19\% &19\% &19\% \\
%& &  Prune  &  0\% &0\% & 25\% & 25\% & 25\% \\
%\midrule[1pt]
\multirow{ 12}{*}{\makecell{CIFAR-10 \\ (ResNet)}}   &\multirow{ 6}{*}{ \makecell{Clean \\  Accuracy}}&Path connection ($t=0.1$) & 87\% &  83\% &  81\% & 80\%  &  77\% \\
& & Fine-tune & 80\% & 79\% & 78\% &  75\% & 72\%\\
& &Train from scratch &  46\% &  39\% & 32\%   &  28\%   &  18\% \\
&  &Noisy model ($t=0$) &  11\%  &  11\% &  11\% &  11\% &  11\% \\
& & Noisy model ($t=1$) & 10\% & 10\% & 10\% & 10\% & 10\%  \\
& & Prune  &  84\% &  82\%& 81\%& 80\%& 78\% \\
\cline{2-8}
&\multirow{ 6}{*}{ \makecell{Injection \\ Accuracy }}&Path connection  ($t=0.1$)&  0\% &  0\%&  0\%&  0\%&  0\%\\
&  & Fine-tune & 0\% & 0\% & 0\% & 25\% &25\% \\
&  & Train from scratch &    0\% &  0\%&  0\%&  0\%&  0\%\\
& &  Noisy model ($t=0$) &  33\% &  33\% &  33\% &  33\% &  33\% \\
& &  Noisy model ($t=1$) & 26\%  & 26\% & 26\% & 26\% & 26\%\\
& & Prune  &  0 \% & 0\%& 0\% & 25\%& 25\%\\
\midrule[1pt]
\multirow{ 12}{*}{\makecell{SVHN \\ (VGG)}}   &\multirow{ 6}{*}{ \makecell{Clean \\  Accuracy} }&Path connection  ($t=0.1$)&    96\%  & 94\%  & 93\% &  91\% & 90\% \\
& & Fine-tune &  94\%  & 93\%  & 92\% &90\%  & 81\% \\
& & Train from scratch & 88\% & 82\% & 76\%  & 68\% & 28\% \\
&& Noisy model ($t=0$) &  33\%  & 33\% & 33\% & 33\% & 33\%\\
& & Noisy model ($t=1$) &  37\% &  37\% &  37\% &  37\% &  37\% \\
& & Prune  &  94\%& 93\% & 92\%& 90\%& 89\% \\
\cline{2-8}
&\multirow{ 6}{*}{ \makecell{Injection \\ Accuracy }}&Path connection  ($t=0.1$) &  0\%  &  0\%  &  0\%  &  0\%  &  0\% \\
& & Fine-tune &  0\% & 0\% & 25\% & 0\% & 25\% \\
& &Train from scratch &   0\%  &  0\%  &  0\%  &  0\%  &  0\% \\
& & Noisy model ($t=0$) & 24\%  & 24\%  & 24\%  & 24\%  & 24\%   \\
& & Noisy model ($t=1$) & 29\% & 29\% & 29\% & 29\% & 29\% \\
& & Prune  &   0\% & 25\% &0\% & 25\% & 25\%\\
%\midrule[1pt]
%\multirow{ 12}{*}{\makecell{SVHN \\ (ResNet)}}   &\multirow{ 6}{*}{ \makecell{Clean \\  Accuracy} }&Path connection  ($t=0.1$)&  96\%  & 94\%  & 92\% & 91\%  & 90\%   \\
%& & Fine-tune &  94\% & 93\% & 91\% & 89\%  & 88\%   \\
% & & Train from scratch &  90\%  &  83\%  & 75\% &  61\%  & 21\%    \\
%&  & Noisy model ($t=0$) &  11\%  &  11\% &  11\% &  11\% &  11\%   \\
%&  & Noisy model ($t=1$) &  11\%  &  11\% &  11\% &  11\% &  11\%  \\
%& & Prune & 95\% & 93\% & 92\% &  90\%  & 89\% \\
%\cline{2-8}
%&\multirow{ 6}{*}{ \makecell{Injection \\ Accuracy }}&Path connection  ($t=0.1$) & 0\% &  0\%&  0\%&  0\%&  0\%\\
%&  & Fine-tune & 0\% & 25\%  & 0\% & 25\% & 25\%  \\
%&  & Train from scratch &   0\%&  0\%&  0\%&  0\%&  0\%\\
%& & Noisy model ($t=0$) &  28\% &  28\%&  28\%&  28\%&  28\%\\
%& & Noisy model ($t=1$) &  18\% &  18\% &  18\% &  18\% &  18\%  \\
%& & Prune & 0\% & 0\% & 25\% &  0\%  & 25\% \\
\bottomrule[1pt]
\end{tabular}
\end{threeparttable}
}
\end{center}
\vspace{-4mm}
\end{table}

\section{Extensions of Path Connection to Different Settings}
\label{appen_one_model}
In Sections \ref{subsec_backdoor} and \ref{subsec_injection}, we consider the scenario where the two given models are tampered using in the same way -- using the same poisoned dataset for backdoor attack and the same targeted images for error-injection attack. Here we discuss how to apply path connection to the case when only one tampered model is given. In addition, we show that the the resilient effect of path connection against these attacks still hold when the two given models are tampered in a different way.

\paragraph{Backdoor and error-injection attacks given one tampered model}

We propose to first fine-tune the model using the bonafide data and then connect the original model with the fine-tuned model. The fine-tuning process uses 2000 images with 100 epochs.  The path connection results are shown in Figure \ref{fig: backdoor_one_tampered}. The start model is a backdoored model with high accuracy for triggered images. The end model is a fine-tuned model where the triggers do not have any effects to cause any misclassification. We can see that through path connection, we can eliminate the influence of the triggers quickly in some cases. For example, with 250 images, the error rate of triggered images reaches 100\% at $t=0.25$ while the clean accuracy at this point is lower than the fine-tuned model at $t=1$.

Similarly, for the case of error-injection attack, we first fine-tune the model using the bonafide data and then connect the original model with the fine-tuned model. We follow the same settings of the backdoor attack for the finetuning and path connection. The performance of the one tampered model case is shown in Figure \ref{fig: injection_one_tempered_two_tempered}(a). We can see that the effects of injected errors can be eliminated quickly through path connection while the clean accuracy is kept high.

 \begin{figure}[h]    
 \centering
\begin{tabular}{p{2.5in}p{1.8in}}
\parbox{2.5in}{\centering Single target attack}
 & \parbox{2.0in}{\centering Legend}\\
\includegraphics[align=c,width=2.5in]{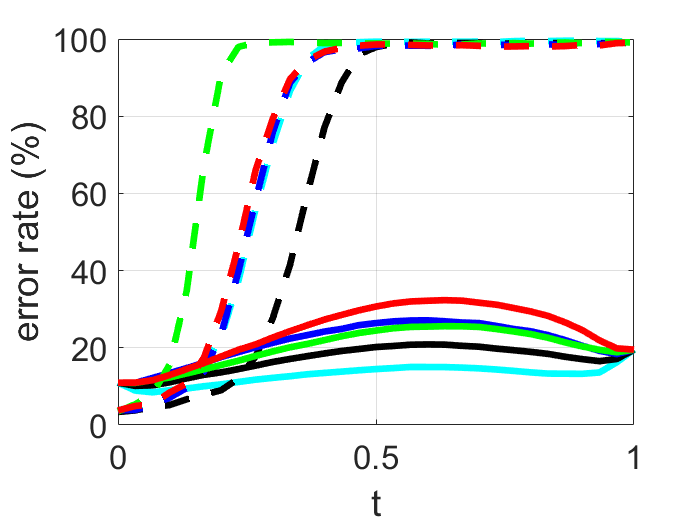}  &
\includegraphics[align=c,width=1.8in]{figs_new/legend.png} 
\end{tabular}
\caption{Error rate under backdoor attack on path connection for CIFAR-10 (VGG). The error rate of clean/triggered samples means the standard-test-error/attack-failure-rate, respectively.} 
\label{fig: backdoor_one_tampered}
\end{figure}

\begin{figure}[h]
\hspace{-5.5mm}
 \centering
\begin{tabular}{p{1.9in}p{1.9in}p{1.2in}}
\parbox{1.9in}{\centering One tampered model}
 & \parbox{1.9in}{\centering  Two tampered models with different injected errors}
  & \parbox{1.2in}{\centering  Legend}\\
\includegraphics[align=c,width=1.9in]{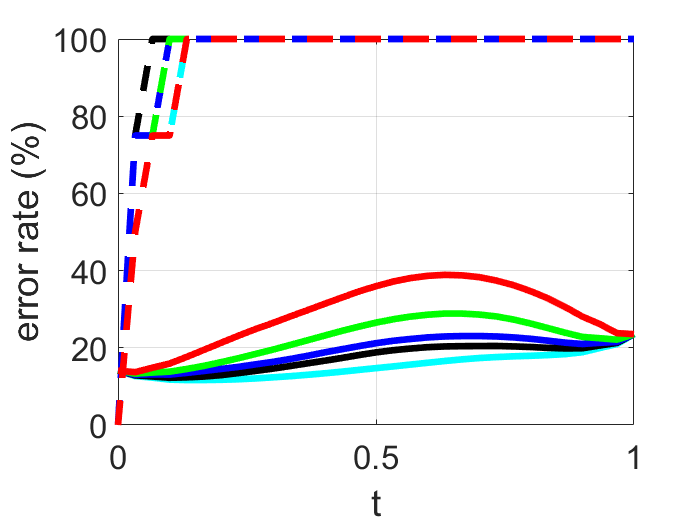}  &
\includegraphics[align=c,width=1.9in]{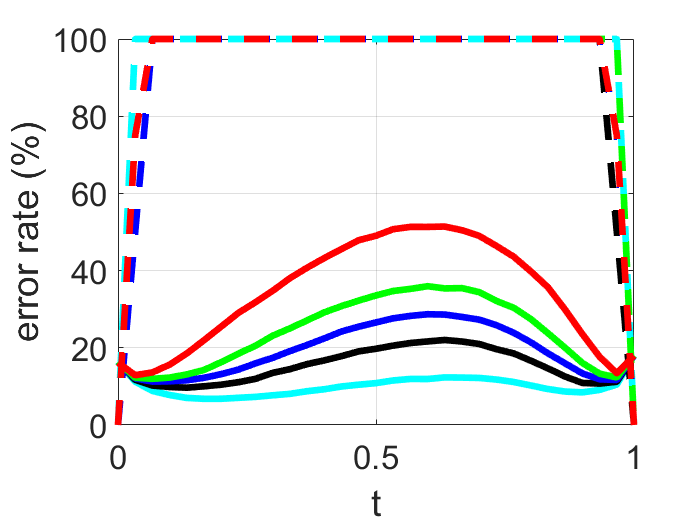} & 
\includegraphics[align=c,width=1.35in]{figs_new/legend.png} 
\end{tabular}
\vspace{-2mm}
\caption{Error rate against error-injection attack on the connection path for CIFAR-10 (VGG). 
The error rate of clean/targeted samples means standard-test-error/attack-failure-rate, respectively.} 
\label{fig: injection_one_tempered_two_tempered}
\vspace{-4mm}
\end{figure}

\paragraph{Backdoor and error-injection attacks given two differently tampered models}

 If the  given two backdoored models are trained with different poisoned datasets (e.g. different number of poisoned  images), the path connection method works as well in this case. 
 We train two backdoored models by poisoning 30\% and 10\% of the training dataset, respectively. The  performance of the path connection between the two models are shown in Figure \ref{fig: backdoor_cifar_VGG_2_tampered}. We can see that the connection can quickly remove the adversarial effects of backdoor attacks.
 
 If the  two  given models with injected errors are trained with different settings (e.g. different total number of training images to inject the errors), the path connection method works as well in this case.  For the start and end models,  the number of injected errors is set to 4. The number of images with the same classification requirement is set to 996 for the start model, and 1496 for the end model, respectively. The performance of path connection is shown in Figure \ref{fig: injection_one_tempered_two_tempered}(b). We can observe that it is able to  obtain a robust model with high clean accuracy.

 \begin{figure}[h]    
 \centering
\begin{tabular}{p{1.9in}p{1.9in}p{1.2in}}
\parbox{1.9in}{\centering Single-target attack}
 & \parbox{1.9in}{\centering All-targets attack}
  & \parbox{1.2in}{\centering Legend}\\
\includegraphics[align=c,width=1.9in]{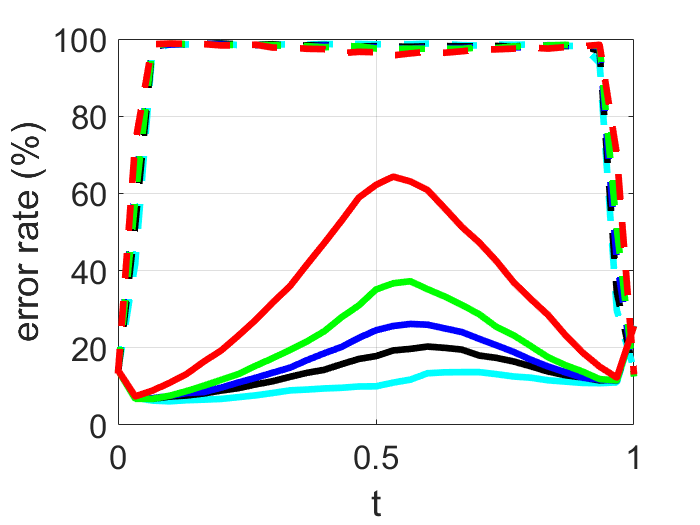}  &
\includegraphics[align=c,width=1.9in]{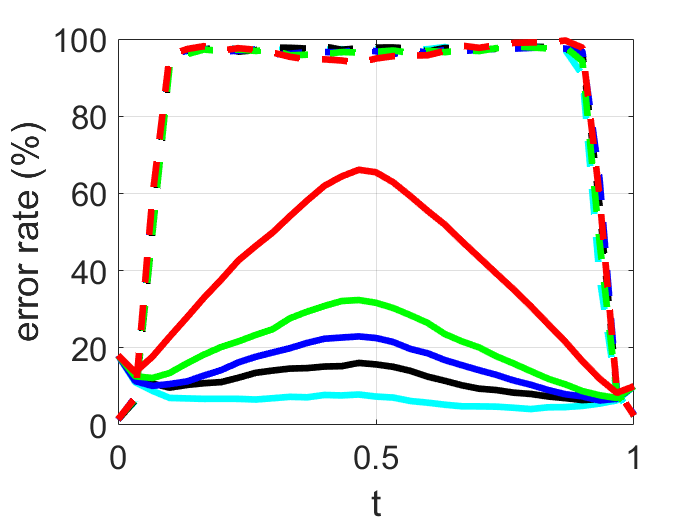} & 
\includegraphics[align=c,width=1.2in]{figs_new/legend.png} 

\end{tabular}
\caption{Error rate against backdoor attack on the connection path for CIFAR-10 (VGG). 
The error rate of clean/triggered samples means the standard-test-error/attack-failure-rate, respectively.} 
\label{fig: backdoor_cifar_VGG_2_tampered}
\end{figure}

%\paragraph{Fine-tuning using bonafide data}

%\paragraph{Training from scratch using bonafide data.}

%\paragraph{Random Gaussian Perturbation to Model Weights}

\section{Model Weight Space Exploration}
\label{appen_noisy_model}
To explore the model weight space, we add Gaussian noise to the weights of a backdoored model to generate $1000$ noisy models of the backdoored model at $t=0$. The standard normal Gaussian noise has zero mean and a standard  deviation of the absolute difference between the two end models on the path.  The distribution of clean accuracy and backdoor accuracy of these noisy models are reported in Figure \ref{fig: noise_models} (a). The results show that the noisy models are not ideal for attack mitigation and model repairing since they suffer from low clean accuracy and high attack success rate. In other words, it is unlikely, if not impossible, to find good models by chance.
%By adding noise to the backdoored model, it is nearly impossible to defend the backdoor attack. 
We highlight that finding a path robust to backdoor attack between two backdoored models is highly non-intuitive, considering the high failure rate of adding noise. 
We can observe similar phenomenon for the injection attack in Figure \ref{fig: noise_models} (b).

\begin{figure}[h]   
 \centering
\begin{tabular}{p{2.5in}p{2.5in}}
 %\parbox{2.5in}{\footnotesize \centering Clean Accuracy} & 
 %\parbox{2.5in}{\footnotesize \centering  Accuracy} 
 %\\ 
\includegraphics[align=c,width=2.5in]{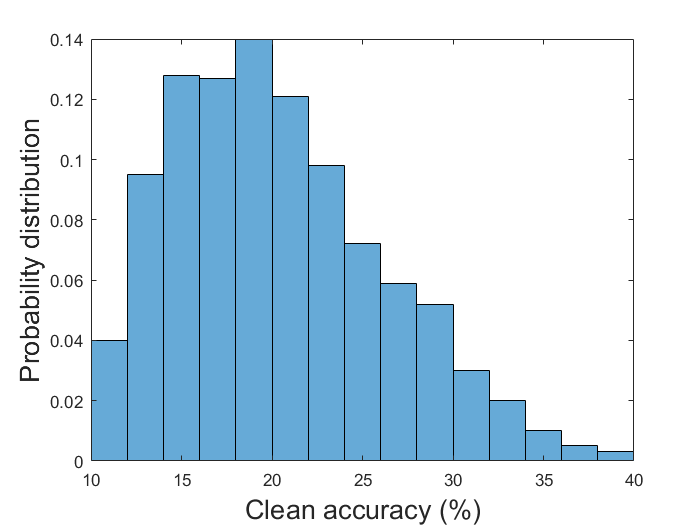}  &
\includegraphics[align=c,width=2.5in]{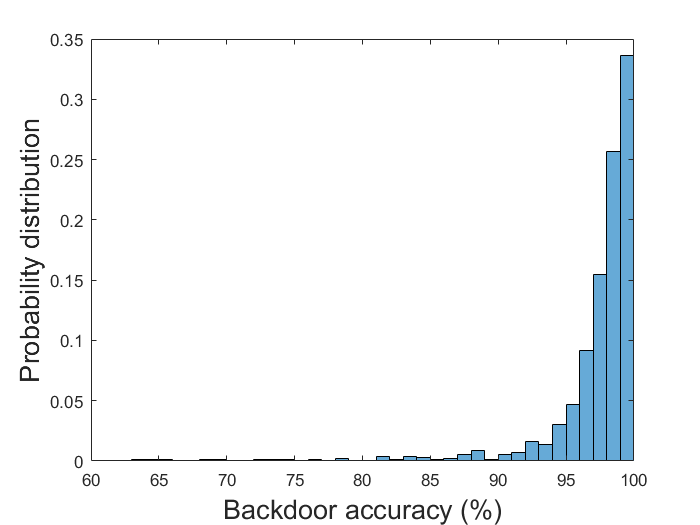} 
\\
 \parbox{5in}{ \centering (a) backdoor attack} 
 \\
\includegraphics[align=c,width=2.5in]{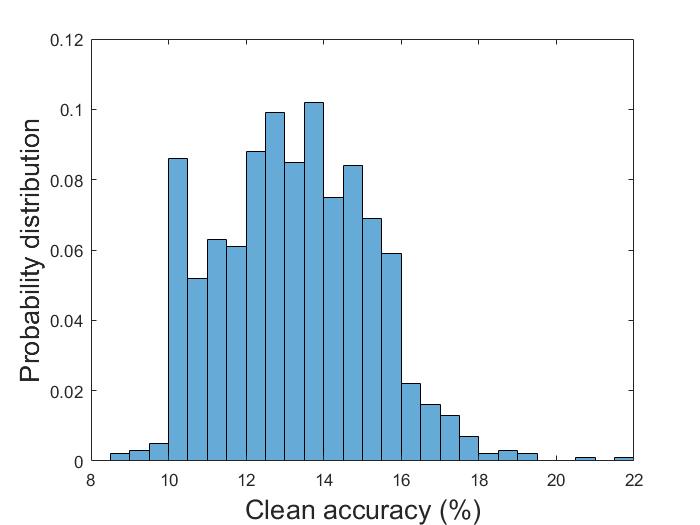} &
\includegraphics[align=c,width=2.5in]{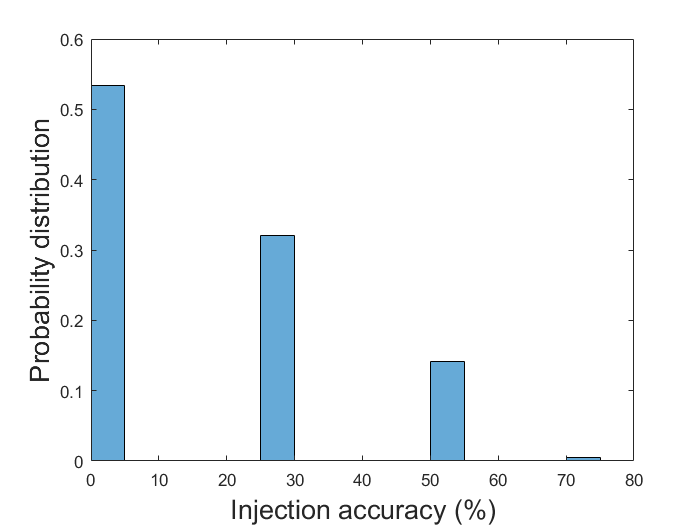} 
\\
 \parbox{5in}{ \centering (b) error-injection attack}  \\
\end{tabular}
\caption{Clean and attack accuracy distribution for 1000 noisy models.}
\label{fig: noise_models}
\end{figure}

\section{Data Similarity of Input Gradients}
\label{appen_similarity}
To provide a technical explanation on why our path connection approach can mitigate backdoor and injection attacks, we compare the similarity of input gradients  between the models on the  connection path ($t \in (0,1)$) and the end models ($t=0$ or $t=1$) in terms of clean data and tampered data. The rationale of inspecting input gradient similarity can be explained using the first-order approximation of the training loss function. Let $l(x|w_t)$ denote the training loss function of a data input $x$ given a model $w_t$ on the connection path, where $t\in[0,1]$. Then the first-order Taylor series approximation on $l(x|w_t)$ with respect to a data sample $x_i$ is $l(x|w_t)=l(x_i|w_t)+\nabla_{x}l(x_i|w_t)^T (x - x_i)$, where $\nabla_{x}l(x_i|w_t)$ is the gradient of $l(x|w_t)$ when $x=x_i$. Based on the flat loss of mode connectivity, we can further assume $l(x_i|w_t)$ is a constant for any $t\in[0,1]$. Therefore, for the same data sample $x_i$, the model $w_t$ ($t\in(0,1)$) will behave similarly as the end model $w_0$ (or $w_1$) if its input gradient $\nabla_{x}l(x_i|w_t)$ is similar to $\nabla_{x}l(x_i|w_0)$ (or $\nabla_{x}l(x_i|w_1)$).
%To show the difference between the models on the path connection ($t \in (0,1)$) and the end models ($t=0$ or $t=1$), 
Figure \ref{fig: similarity2} shows the average cosine similarity distance of the input gradients between the models on the path and the end models for the backdoor attack and the injection attack. 
% \textcolor{red}{
% We can observe that the clean inputs similarity is higher than the backdoored inputs similarity, demonstrating that  the models on the path are uncorrelated or less correlated with the backdoored model for the backdoored images than the clean images. Thus for the models on the path, the backdoored images suffer from more significant accuracy degradation than the clean images. We can observe the same phenomenon for the injection attack. }
The pairwise similarity metric for each data sample is defined as $ m = | s -1 |/2 $, where $s$ is the  cosine similarity of the input gradients between the models on the path and the end models. Smaller $m$ means higher similarity of the input gradients.  Comparing the minimum of the solid and dashed curves of different colors on the path respectively, which corresponds to the similarity to either one of the two end models,  we find that the similarity of clean data are consistently higher than that of tampered data. Therefore, the sanitized models on the path can maintain similar clean data accuracy and simultaneously mitigate adversarial effects as these models are dissimilar to the end models for the tampered data.
%The average similarity metric is shown in Figure \ref{fig: similarity2}.

\begin{figure}[h]   
 \centering
\begin{tabular}{p{2.5in}p{2.5in}}
\parbox{2.5in}{\footnotesize \centering Path for the backdoor attack}
 & \parbox{2.5in}{\footnotesize \centering  Path for the injection attack}
  \\
\includegraphics[align=c,width=2.5in]{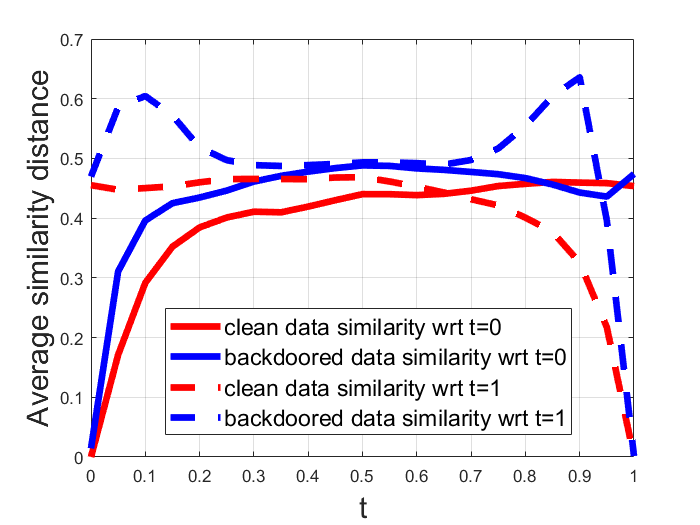}  &
\includegraphics[align=c,width=2.5in]{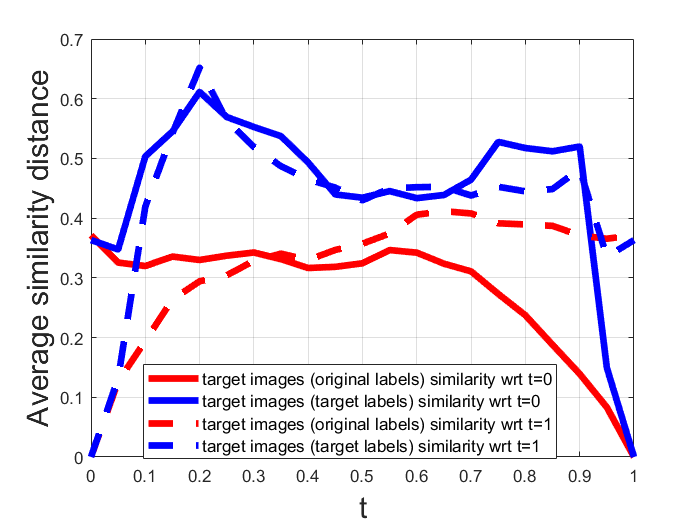} 
\\
\end{tabular}
\caption{  Similarity distance between the models on the path and the two end models for CIFAR-10 (VGG). Smaller distance value means higher similarity.}
\label{fig: similarity2}
\end{figure}

\section{Adaptive Attacks}
\label{appen_adaptive}

We consider the advanced attack setting where the attacker knows path connection is used for defense but  cannot compromise the bonafide data. 
Furthermore, we allow the attacker to  use the \textit{same} path training loss function as the defender.

\textbf{Backdoor attack.}  To attempt breaking path connection, the attacker first separately trains two  backdoored models with one poisoned dataset. 
Then the attacker uses the same poisoned dataset to connect the two models and hence compromises the models on the path. %Note that during training the path connection, the start and end models are not fixed and they are also basically finetuned by the poisoned dataset. 
\textcolor{black}{Note that when training this tampered path, in addition to learning the path parameter $\theta$, the start and end models are not fixed and they are fine-tuned by the poisoned dataset.}
Next the adversary releases the start and end models $(t=0,1)$ from this tampered path. 
Finally, the defender trains a path from these two models with bonafide data. 
%The For the user/defender, after obtaining the start and end models, it will build a connection between them with bonafide data. 
We conduct the advanced (path-aware) single-target backdoor attack experiments on CIFAR-10 (VGG) by poisoning 10\% of images in the training set with a trigger. Figure \ref{fig: backdoor_cifar_VGG3} (a) shows the entire path has been successfully compromised \textcolor{black}{due to the attacker's poisoned path training data}, yielding less than 5\% attack error rate on 10000 test samples. Figure \ref{fig: backdoor_cifar_VGG3} (b) shows the defense performance with different number of clean data to connect the two specific  models released by the attacker. We find that path connection is still resilient to this advanced attack and most models on the path (e.g. $t\in [0.25,0.75]$) can be repaired. Although this advanced attack indeed decreases the portion of robust models on the path, it still does not break our defense.
In Table \ref{tab: backdoor_comparison2}, we also compare the generalization and defense performance and demonstrated that path connection outperforms other approaches.  Moreover,
in the scenario where two tampered models are close in the parameter space (in the extreme case they are identical), we can leverage the proposed path connection method  with one tampered model to alleviate the adversarial effects. The details are discussed in the "Extensions" paragraph and Appendix \ref{appen_one_model}. 

\begin{figure}[h]   
 \centering
\begin{tabular}{p{1.9in}p{1.9in}p{1.3in}}
\includegraphics[align=c,width=1.9in]{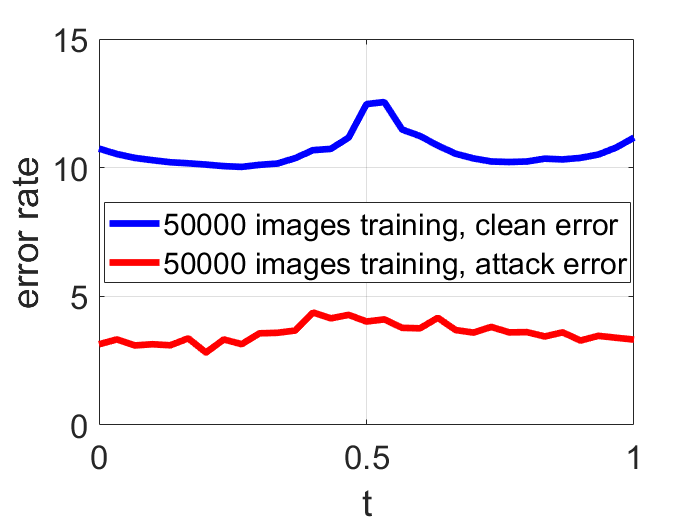}  &
\includegraphics[align=c,width=1.9in]{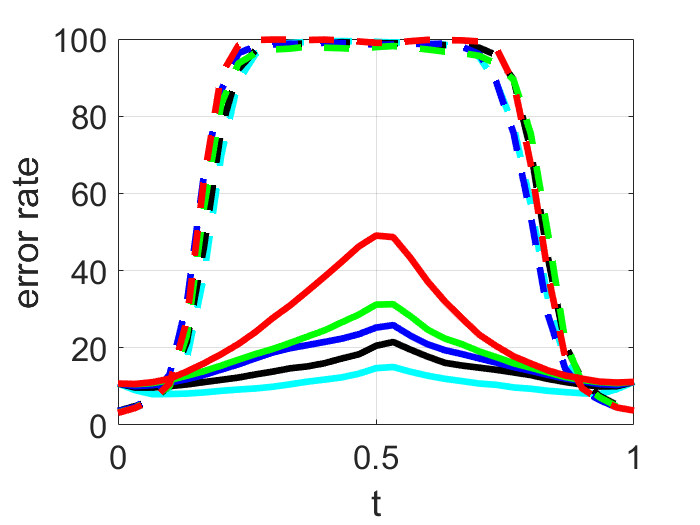} &
\includegraphics[align=c,width=1.3in]{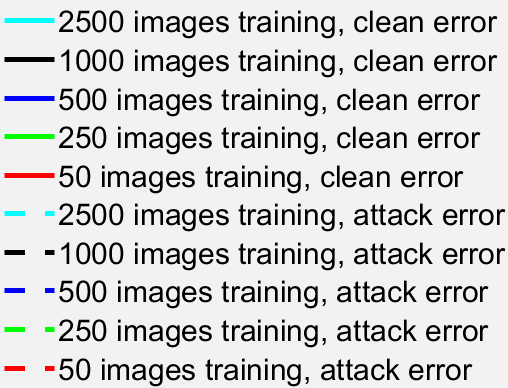} 
\\
\parbox{1.9in}{\footnotesize \centering (a) Path with poisoned data}
 & \parbox{1.9in}{\footnotesize \centering  (b) Path with clean data} & \parbox{1.3in}{\footnotesize \centering  (c) Legend}
  \\
\end{tabular}
\caption{ Error rate against  path-aware single-target backdoor attacks for CIFAR-10 (VGG).}
\label{fig: backdoor_cifar_VGG3}
\end{figure}

\begin{table}[h]
\begin{center}
\caption{Performance against path-aware single-target backdoor attack on CIFAR-10 (VGG).} 
\label{tab: backdoor_comparison2}
\scalebox{0.85}{
\begin{threeparttable}
\begin{tabular}{c|c|c|c|c|c|c|c|c|c|c}
\toprule[1pt]
 & \multicolumn{5}{c|}{ Clean accuracy } & \multicolumn{5}{c}{ Backdoor accuracy }   \\
\hline
Method / Bonafide data size & 2500 & 1000  & 500  &  250 & 50 & 2500 & 1000  & 500  &  250 & 50 \\
\midrule[1pt]
Path connection $(t=0.27)$ & 90\% & 87\% & 83\% & 81\% &  75\% & 3.8\%  & 2.9\% & 3.6\% & 4.2\% & 5.6\% \\
Fine-tune &   88\% & 84\% & 82\% & 80\% & 69\% & 4.1\% &  4.6\% & 4.4\% & 3.9\% & 5.9\% \\
Train from scratch & 58\% & 48\% & 36\% & 28\% & 20\%&   0.6\% &  0.5\% &  0.9\% & 1.7\% & 1.9\%  \\
 Noisy model ($t=0$) &  38\% &  38\% & 38\% & 38\% & 38\% &   91\% & 91\% & 91\% & 91\% & 91\%\\
 Noisy model ($t=1$) &  35\%&  35\%&  35\%&  35\%&  35\% &   86\% &86\% & 86\% & 86\% & 86\% \\
Prune  & 87\%  &  85\% & 83\%  & 81\% & 79\% &   29\% & 48\% & 69\% & 77\% & 81\%  \\
\bottomrule[1pt]
\end{tabular}
\end{threeparttable}
}
\end{center}
\end{table}

\textbf{Error-injection attack.} For the advanced attack in the error-injection setting, the attacker first trains two models with injected errors and then connects the two models with clean data. Then, the attacker tries to inject errors to the three models $w_1$,  $w_2$ and $\theta$. Note that based on Equation (\ref{equ_poly}) or (\ref{eqn_Bezier_quad}), all the models on the path can be expressed as a combination of these three models. Thus, the whole path is expected to be tampered. Next, the attacker releases the start and end models from the ``bad'' path.  Finally, the defender trains a path from these two models with bonafide data.
We conduct the advanced (path-aware) error-injection attack experiments on CIFAR-10 (VGG) by injecting $4$ errors. Figure \ref{fig: backdoor_cifar_VGG4} (a) shows that the entire path has been successfully compromised by the advanced path-aware injection. While the clean error rate is stable across the path, at least $3$ out of
$4$ injected errors (corresponding to 25\% attack error rate) yield successful attacks. 
%The clear error is not high since the attacker has a good starting point. But the  errors has been almost successfully injected into the path with high accuracy. 
After training a path to connect the two released models with the bonafide data, the clean error and attack error are shown in Figure \ref{fig: backdoor_cifar_VGG4} (b). We can observe that path connection can effectively eliminate the injected errors on the path (i.e., high attack error rate).

\begin{figure}[h]   
 \centering
\begin{tabular}{p{1.9in}p{1.9in}p{1.3in}}
\includegraphics[align=c,width=1.9in]{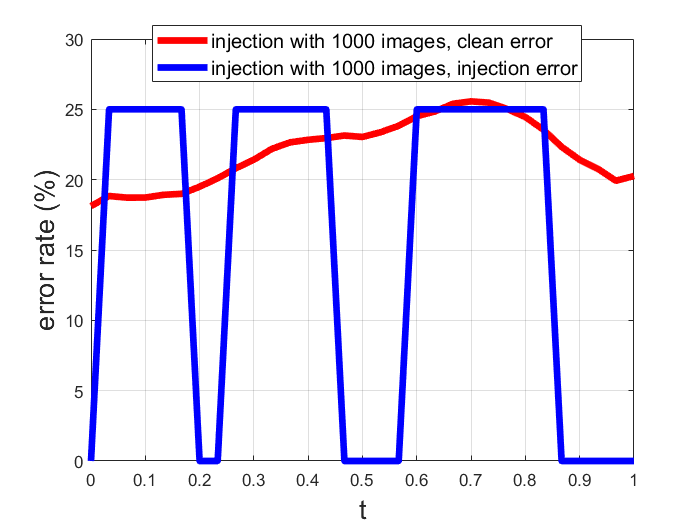}  &
\includegraphics[align=c,width=1.9in]{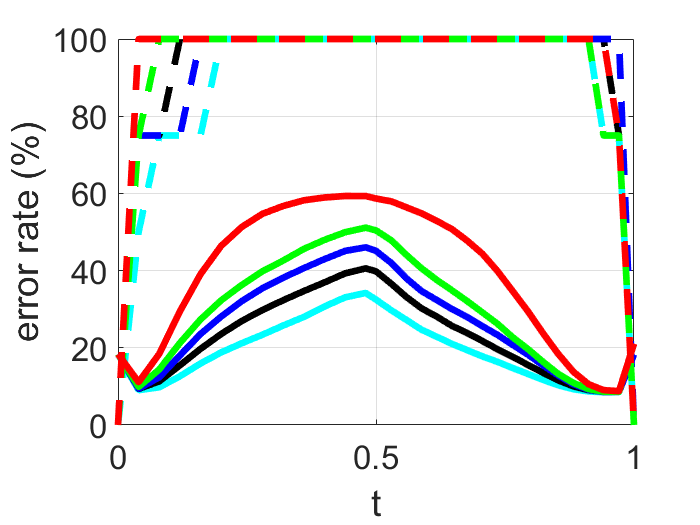} &
\includegraphics[align=c,width=1.3in]{figs_new/rebuttal/legend.png} 
\\
\parbox{1.9in}{\footnotesize \centering (a) Path with injection attack}
 & \parbox{1.9in}{\footnotesize \centering  (b) Path without injection attack} & \parbox{1.3in}{\footnotesize \centering  (c) Legend}
  \\
\end{tabular}
\caption{ Error rate against  path-aware error-injection attacks for CIFAR-10 (VGG).}
\label{fig: backdoor_cifar_VGG4}
\end{figure}

% \begin{figure}[h]   
%  \centering
% \begin{tabular}{p{2.5in}p{2.5in}}
% \parbox{2.5in}{\footnotesize \centering Path for the backdoor attack}
%  & \parbox{2.5in}{\footnotesize \centering  Path for the injection attack}
%   \\
% \includegraphics[align=c,width=2.5in]{figs_new/backdoor_similarity_0_1.png}  &
% \includegraphics[align=c,width=2.5in]{figs_new/injection_similarity_0_1.png} 
% \\
% \end{tabular}
% \caption{ Similarity between the models on the path and the end models  for CIFAR-10 (VGG).}
% \label{fig: similarity}
% \end{figure}

\section{Implementation Details for Evasion Attack and Adversarial Training}
\label{appen_evasion}
%Among them, the adversarial training method XXX can achieve impressing robust performance against evasion attack XXX. To have a better understanding for the robust model,  we investigate the path connection between a non-robust model and robust model obtained through adversarial training. The model architecture is VGG16 and the dataset is CIFAR-10. 
\paragraph{Evasion attack}
After training the connection with the training set for 100 epochs, adversarial examples of the whole test set generated by the $\ell_{\infty}$-norm based PGD method \citep{madry2017towards} with attack strength $\epsilon=8/255$ and 10 attack iterations are used to test the performance of the path connection.

\paragraph{Adversarial training}
For adversarial training, at each model weight update stage, we first generate adversarial examples with the PGD method \citep{madry2017towards}. The attack strength is set to  $\epsilon=8/255$  with 10 attack iterations. After the adversarial 
examples are obtained, we then update the model weights based on the training losses induced by these adversarial examples with their correct labels.

\paragraph{Input Hessian}
As the adversarial examples perform small perturbations around the clean images to increase the robustness loss function, thus incurring a misclassification, it relates to the Hessian matrix of the loss function with reference to the input images. A Hessian matrix (or Hessian) is the second-order partial derivatives of a scalar-valued function, describing the local curvature of a function with many variables. In general,   large Hessian spectrum means the function reaches a sharp minima, thus leading to a more vulnerable model as the input can leave this minima with small distortions. By contrast, in the case of flat minima with small Hessian spectrum, it takes more efforts for the input to leave the minima. 

As we find the robustness loss barrier on the path, we are interested in the evolution of the input Hessian for the models on the path. We uniformly pick the models on the connection and compute the largest eigenvalue of the Hessian w.r.t. to the inputs. To deal with the high dimension difficulties for the Hessian calculation,  the power iteration method \citep{martens2012training} is adopted to compute the largest eigenvalue of the input Hessian with the first-order derivatives obtained by back-propagating. Unless otherwise specified, we continue the power iterations until reaching a relative error of 1E-4. 

For ease of visual comparison,
in Figure \ref{fig: input_hessian_vs_adversarial_loss}, we plot the log value of the largest eigenvalue of the input Hessian together with the error rate and loss for  clean images and adversarial examples. We note that the Pearson correlation coefficient (PCC) is indeed computed using the original largest eigenvalue of input Hessian and robustness loss.
As demonstrated in Figure \ref{fig: input_hessian_vs_adversarial_loss}, the evolution of the input Hessian is very similar to the change of the  loss of adversarial examples. As we can see, the largest eigenvalue on the path does not necessarily have a high correlation with the error rate of the clean images in the training set and test set or the adversarial examples. Instead, it seems to be highly correlated with the loss of adversarial examples as they share very similar shapes.  This inspires us to explore the relationship between the largest eigenvalue of the input Hessian and the robustness loss.

\section{Proof of Proposition \ref{prop_correlation}}
\label{appen_prop_correlation}
For simplicity, here we drop the the notation dependency on the input data sample $x$. It also suffices to consider two models $w:=w(t)$ and $w+\Delta w:=w(t+\Delta t)$ for some small $\Delta t$ on the path for our analysis. We begin by proving the following lemma.
\begin{lemma}
\label{lemma_norm}
Given assumption (a), for any vector norm $\|\cdot\|$, for any data sample $x$,
\begin{equation}
 \lVert \nabla _xl\left( w+\varDelta w,x \right) \rVert -\lVert \nabla _xl\left( w,x \right) \rVert  = 0
\end{equation}
\end{lemma}
\begin{proof}
With assumption (a), we have $l\left( w+\varDelta w,x \right) = l\left( w,x \right)=\textnormal{const}$. Differentiating at both sides with respect to $x$ and then applying the vector norm gives the lemma.
\end{proof}

As the adversarial perturbation is generated using the PGD method \citep{madry2017towards} with an $\epsilon$-$\ell_\infty$ ball constraint, the optimal first-order solution of the perturbation is 
\begin{equation}
\delta^*=\epsilon \cdot \frac{\nabla f\left( x \right)}{\lVert \nabla f\left( x \right) \rVert}
\label{eqn_first_order}
\end{equation}

Assume the PGD method can well approximate the robustness loss, 
the difference of the robustness losses of the two models $w$ and $w+\Delta w$ on the path can be represented as
%begin{equation}
{ \small
\begin{align}
&\underset{\|\delta\|  \leq \epsilon}{\max}\,\,l\left( w+\varDelta w,x+\delta \right) \,\,-\,\,\underset{\|\delta\| \leq \epsilon}{\max}\,\,l\left( w,x+\delta \right) \,\, \nonumber
\\
& =l\left( w+\varDelta w,x+\delta_{w + \Delta w}^* \right) -l\left( w,x+\delta_{w}^* \right)
\\ \nonumber
& = \left[ l\left( w+\varDelta w,x \right) +\epsilon \frac{\nabla _xl\left( w+\varDelta w,x \right) ^T\nabla _xl\left( w+\varDelta w,x \right)}{\lVert \nabla _xl\left( w+\varDelta w,x \right) \rVert}+\frac{\epsilon ^2}{2}\frac{\nabla _xl\left( w+\varDelta w,x \right) ^T}{\lVert \nabla _xl\left( w+\varDelta w,x \right) \rVert}H_{ w+\varDelta w}\frac{\nabla _xl\left( w+\varDelta w,x \right)}{\lVert \nabla _xl\left( w+\varDelta w,x \right) \rVert} \right]
\label{eqn_p1}
\\
&-\left[ l\left( w,x \right) +\epsilon \frac{\nabla _xl\left( w,x \right) ^T\nabla _xl\left( w,x \right)}{\lVert \nabla _xl\left( w,x \right) \rVert}+\frac{\epsilon ^2}{2}\frac{\nabla _xl\left( w,x \right) ^T}{\lVert \nabla _xl\left( w,x \right) \rVert}H_{w}\frac{\nabla _xl\left( w,x \right)}{\lVert \nabla _xl\left( w,x \right) \rVert} \right] + O(\epsilon^3) 
\\ \nonumber
&\approx l\left( w+\varDelta w,x \right) -l\left( w,x \right) +\epsilon \left[ \lVert \nabla _xl\left( w+\varDelta w,x \right) \rVert -\lVert \nabla _xl\left( w,x \right) \rVert \right] 
\\
&+\frac{\epsilon ^2}{2}\left[ \frac{\nabla _xl\left( w+\varDelta w,x \right) ^T}{\lVert \nabla _xl\left( w+\varDelta w,x \right) \rVert}H_{w+\varDelta w}\frac{\nabla _xl\left( w+\varDelta w,x \right)}{\lVert \nabla _xl\left( w+\varDelta w,x \right) \rVert}-\frac{\nabla _xl\left( w,x \right) ^T}{\lVert \nabla _xl\left( w,x \right) \rVert}H_{w}\frac{\nabla _xl\left( w,x \right)}{\lVert \nabla _xl\left( w,x \right) \rVert} \right] 
\label{eqn_p2}
\\
&=\frac{\epsilon ^2}{2}\left[ \frac{\nabla _xl\left( w+\varDelta w,x \right) ^T}{\lVert \nabla _xl\left( w+\varDelta w,x \right) \rVert}H_{w+\varDelta w}\frac{\nabla _xl\left( w+\varDelta w,x \right)}{\lVert \nabla _xl\left( w+\varDelta w,x \right) \rVert}-\frac{\nabla _xl\left( w,x \right) ^T}{\lVert \nabla _xl\left( w,x \right) \rVert}H_{w}\frac{\nabla _xl\left( w,x \right)}{\lVert \nabla _xl\left( w,x \right) \rVert} \right] 
\label{eqn_diff}
\end{align}
%\end{equation}
}

The first equality uses the definition of optimal first-order solution of the robustness loss. Using \eqref{eqn_first_order} and Taylor series expansion on $\ell(\cdot,\cdot)$ with respect to its second argument 
gives \eqref{eqn_p1}. Rearranging \eqref{eqn_p1} and using Assumption (b) (i.e., ignoring the $O(\epsilon^3)$ term in \eqref{eqn_p1}) gives \eqref{eqn_p2}.
Finally, Based on \eqref{eqn_p2}, the term $l\left( w+\varDelta w,x \right) -l\left( w,x \right)=0$ due to Assumption (a). Then, applying Lemma \ref{lemma_norm}  to \eqref{eqn_p2} gives \eqref{eqn_diff}. 

Without loss of generality, the following analysis assumes $\|\cdot\|$ to be the Euclidean norm and consider the case when $\frac{\nabla _xl\left( w,x \right)^T v}{\|\nabla _xl\left( w,x \right)\|} = c$. The results can generalize to other norms
since the correlation is invariant to scaling factors.
With assumption (c), for any input Hessian $H$ and its largest eigenvector $v$, we have
\begin{align}
\label{eqn_lower}
     &\frac{\nabla _xl\left( w,x \right) ^T}{\lVert \nabla _xl\left( w,x \right) \rVert}H\frac{\nabla _xl\left( w,x \right)}{\lVert \nabla _xl\left( w,x \right) \rVert}  \nonumber \\
     &=  \left( \frac{\nabla _xl\left( w,x \right)}{\lVert \nabla _xl\left( w,x \right) \rVert} -v + v \right)^T H \left( \frac{\nabla _xl\left( w,x \right)}{\lVert \nabla _xl\left( w,x \right) \rVert } -v + v \right) \nonumber \\
     &= \left( \frac{\nabla _xl\left( w,x \right)}{\lVert \nabla _xl\left( w,x \right) \rVert} -v \right) ^T H \left( \frac{\nabla _xl\left( w,x \right)}{\lVert \nabla _xl\left( w,x \right) \rVert } -v  \right) + 2 \frac{\nabla _xl\left( w,x \right) ^T}{\lVert \nabla _xl\left( w,x \right) \rVert}H v - \lambda_{\max} \nonumber
     \\
     & = (2c-1) \lambda_{\max} + h_{x},
\end{align}
where $h_{x}:=\left( \frac{\nabla _xl\left( w,x \right)}{\lVert \nabla _xl\left( w,x \right) \rVert} -v \right) ^T H \left( \frac{\nabla _xl\left( w,x \right)}{\lVert \nabla _xl\left( w,x \right) \rVert } -v  \right)$ and the last equality uses the fact that $H v = \lambda_{\max} \cdot v$ and  assumption (c). Note that the lower bound is tight when $v$ and $\nabla _xl\left( w,x \right)$ are collinear, i.e.,
$v= \frac{\nabla _xl\left( w,x \right) ^T}{\lVert \nabla _xl\left( w,x \right) \rVert}$, as in this case $h_x=0$ and $c=1$. On the other hand, using the definition of largest eigenvalue gives an upper bound
\begin{equation}
\label{eqn_upper}
\frac{\nabla _xl\left( w+\varDelta w,x \right) ^T}{\lVert \nabla _xl\left( w+\varDelta w,x \right) \rVert}H\frac{\nabla _xl\left( w+\varDelta w,x \right)}{\lVert \nabla _xl\left( w+\varDelta w,x \right) \rVert}\leq \lambda _{\max}.
\end{equation}

Finally, applying \eqref{eqn_upper} and \eqref{eqn_lower} to \eqref{eqn_diff},
we have 
%{ \small
\begin{align}
    \underset{\|\delta\|  \leq \epsilon}{\max}\,\,l\left( w+\varDelta w,x+\delta \right) \,\,-\,\,\underset{\|\delta\| \leq \epsilon}{\max}\,\,l\left( w,x+\delta \right) \,\,  \leq  \frac{\epsilon^2}{2} \left[\lambda_{\max}(t+\Delta t) - (2c-1)\lambda_{\max}(t) - h_x(t)\right]
\end{align}
%}
and 
%{ \small
\begin{align}
    \underset{\|\delta\|  \leq \epsilon}{\max}\,\,l\left( w+\varDelta w,x+\delta \right) \,\,-\,\,\underset{\|\delta\| \leq \epsilon}{\max}\,\,l\left( w,x+\delta \right) \,\,  \geq  \frac{\epsilon^2}{2} \left[(2c-1)\lambda_{\max}(t+\Delta t) - \lambda_{\max}(t) + h_x(t+\Delta t)\right]
\end{align} %}

As $v$ becomes more aligned with $\nabla _xl\left( w,x \right)$,
$c$ increases toward 1 and $h_x$ decreases toward 0, implying $2c-1 \approx 1$. Therefore, in this case we have $ \underset{\|\delta\|  \leq \epsilon}{\max}\,\,l\left( w+\varDelta w,x+\delta \right) \,\,-\,\,\underset{\|\delta\| \leq \epsilon}{\max}\,\,l\left( w,x+\delta \right) \,\, $ as a linear function of
$\lambda_{\max}(t+\Delta t) - \lambda_{\max}(t)$ and $\underset{\|\delta\|  \leq \epsilon}{\max}\,\,l\left( w(t),x+\delta \right) \,\, \sim \lambda_{\max}(t)$.

If the higher order term $O(\epsilon^3)$ in \eqref{eqn_p1} is non-negligible, which means Assumption (b) does not hold, then the following analysis will have an extra offset term of the order $O(\epsilon)$ as $c$ increases toward 1, i.e., $\underset{\|\delta\|  \leq \epsilon}{\max}\,\,l\left( w(t),x+\delta \right) \,\, \sim \lambda_{\max}(t) + O(\epsilon)$.

\section{Robust Connection and Model Ensembling}
\label{appen_robust_conn}

\paragraph{Robust connection}
To train a robust path, we minimize the expected maximum loss on the path like the following,
\begin{equation}
\min_{\theta} {E_{t \sim U(0,1)}}\left[ {\mathop {\max }\limits_{{\bf{\delta }} \in S} L({\phi _\theta }(t),{\bf{x}} + {\bf{\delta }})} \right]
\end{equation}
where
\begin{equation}
S = \left\{ {{\bf{\delta }}\left| {{\bf{x}} + {\bf{\delta }} \in {{\left[ {0,1} \right]}^d},{{\left\| {\bf{\delta }} \right\|}_\infty } \leq \epsilon } \right.} \right\}
\end{equation}

\begin{figure}[t] 
 \centering
\begin{tabular}{p{0.1in}p{1.25in}p{1.25in}p{1.25in}p{0.85in}}
 & \parbox{1.25in}{\centering \footnotesize Connection of regular models (PCC=0.98)} &  
\parbox{1.25in}{\centering \footnotesize  Connection of regular and adversarially-trained models (PCC=0.98)}  &  
\parbox{1.25in}{\centering \footnotesize Connection of adversarially-trained models (PCC=0.94) }  & 
\parbox{0.85in}{\centering \footnotesize Legend  } \\
\vspace{-0.4in} \rotatebox{90}{\parbox{0.6in}{\centering \footnotesize loss \& eigenvalue }}  &  \includegraphics[align=c,width=1.25in]{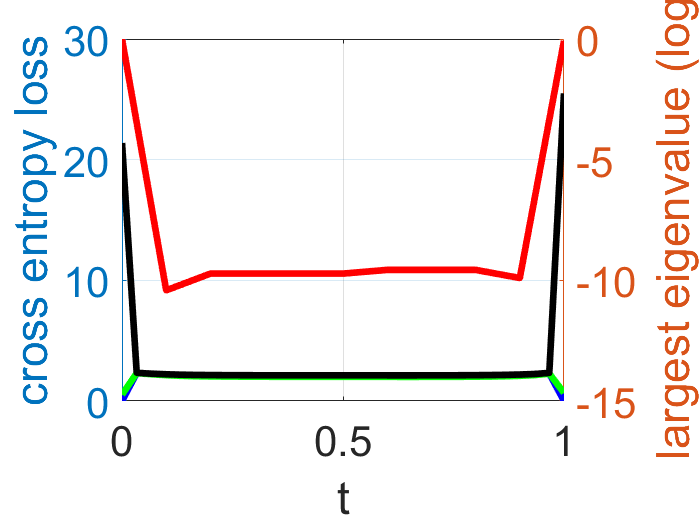} &
\includegraphics[align=c,width=1.25in]{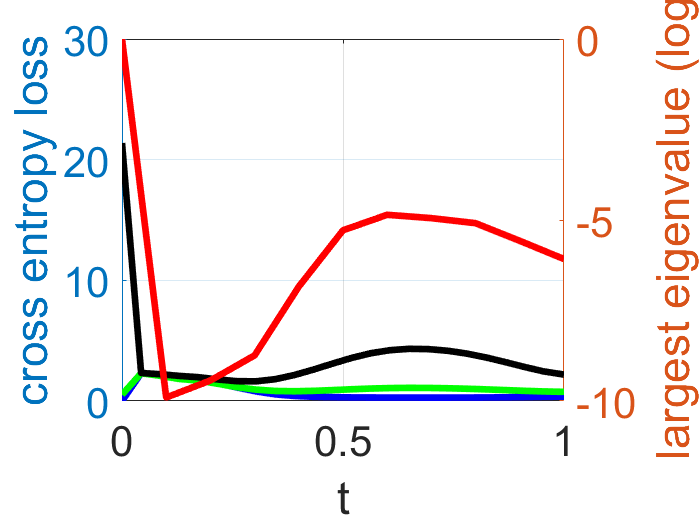}& 
\includegraphics[align=c,width=1.25in]{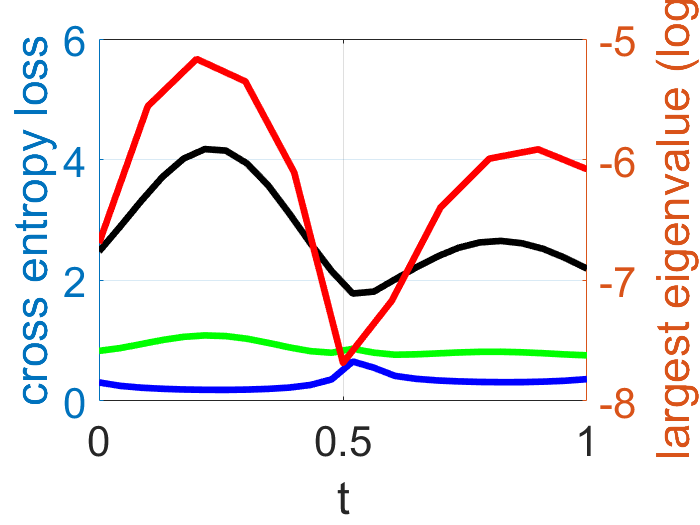} &\includegraphics[align=c,width=0.85in]{figs_new/loss_legend.png} 
\\ 
\vspace{-0.5in} \rotatebox{90}{\parbox{0.9in}{\centering \footnotesize error rate \& \\ attack success rate }} & \includegraphics[align=c,width=1.25in]{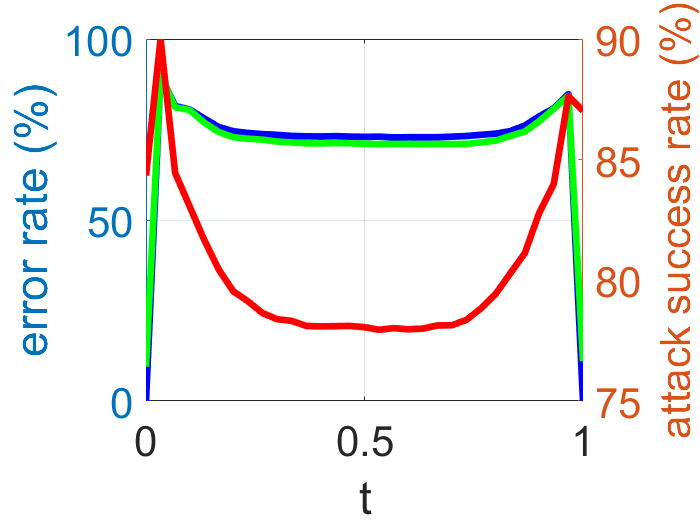} &
\includegraphics[align=c,width=1.25in]{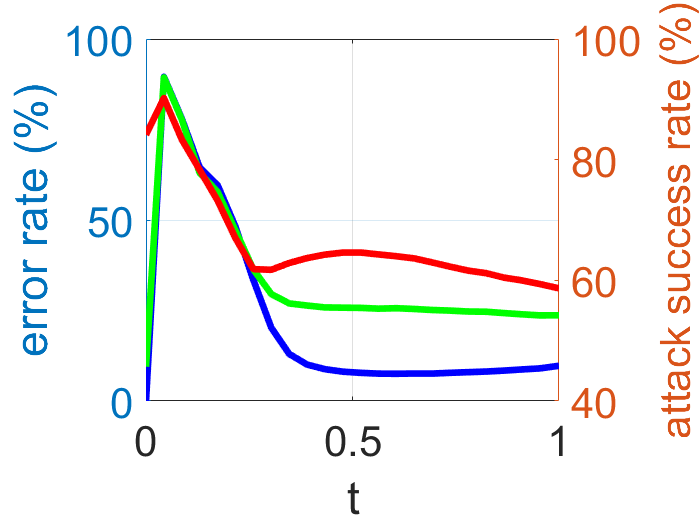} & 
\includegraphics[align=c,width=1.25in]{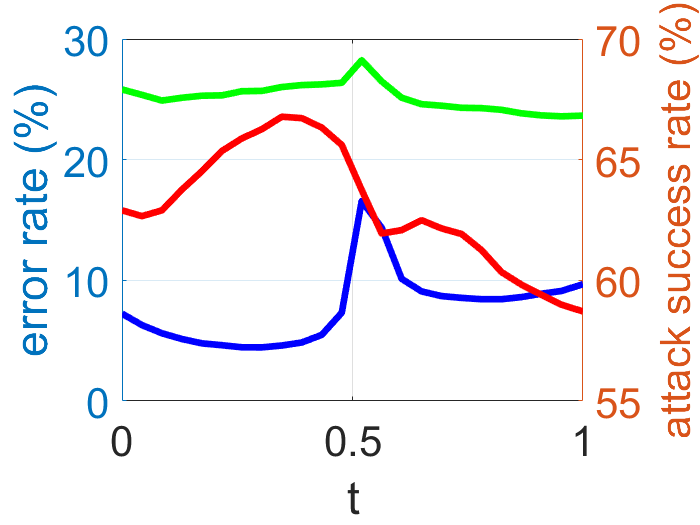}
&\includegraphics[align=c,width=0.85in]{figs_new/error_legend.png}
\end{tabular}
\vspace{-2mm}
\caption{Loss, error rate, attack success rate and largest eigenvalue of input Hessian on the path connecting different model pairs on CIFAR-10 (VGG) using robust connection. The path is obtained with robust training method. The error rate of training/test data means standard training/test error, respectively. There is still a robustness loss barrier between non-robust and robust models, but not there is no robustness loss barrier between robust model pair and non-robust model pair (verified by small loss variance and flat attack success rate on the path). There is also a high correlation
between the robustness loss and the largest eigenvalue of input Hessian, and their Pearson correlation coefficient (PCC) is given in the title of each plot.} 
\label{fig: input_hessian_vs_adversarial_loss_robust_connection}
\vspace{-4mm}
\end{figure}

To solve the problem, we first sample $\tilde t$ from the uniform distribution $U(0,1)$ and we can obtain the model ${\phi _\theta }(\tilde t)$. Based on the model  ${\phi _\theta }(\tilde t)$,  
we  find the perturbation maximizing the loss within the range $S$,
\begin{equation}
\mathop {\max }\limits_{{\bf{\delta }} \in S} L({\phi _\theta }(\tilde  t),{\bf{x}} + {\bf{\delta }})
\end{equation}
We can use projected gradient descent method for the maximization. 
\begin{equation}
{{\bf{\delta }}_{k + 1}} = \mathop \prod \limits_S \left( {{{\bf{\delta }}_k} + \eta  \cdot {\mathop{\rm sgn}} \left( {{\nabla _{\bf{\delta }}}L({\phi _\theta }(\tilde t),{\bf{x}} + {\bf{\delta }})} \right)} \right),
\end{equation}
where $\prod \limits_S$ denotes the projection to the feasible perturbation space $S$, and $\rm sgn(\cdot)$ denotes element-wise sign function taking the value of either $1$ or $-1$.

After finding the perturbation $\hat \delta $ maximizing the loss, we would minimize the expectation. At each iteration, we make a gradient step for $\bm \theta$ as follows,
\begin{equation}
{\bf{\theta }} = {\bf{\theta }} - \eta  \cdot {\nabla _\theta }L({\phi _\theta }(\tilde t),{\bf{x}} + {\bf \hat \delta } )
\end{equation}

We show the robust connection for a pair of non-robust and non-robust models, a pair of non-robust and robust models, and  a pair of robust models in Figure \ref{fig: input_hessian_vs_adversarial_loss_robust_connection}. We can observe that with the robust training method, there is still a robustness loss barrier between the non-robust and robust models. However, for the robust connection of robust model pair and non-robust (regular) model pairs, there is no robustness loss barrier, verified by small loss variance and flat attack success rate on the path. Our results also suggest that there is always a loss barrier between non-robust and robust models, no matter using standard or robust loss functions for path connection, as intuitively the two models are indeed not connected in their respective loss landscapes. Moreover,
the attack success rate of adversarial examples are relatively small on the whole path compared with the robust connection of non-robust and robust models.

\paragraph{Model ensembling}
Here we test the performance of (naive) model ensembling against evasion attacks. Given two untampered and independently trained CIFAR-10 (VGG) models, we first build a regular connection of them. Then we randomly choose models on the path (randomly choose the value of $t$) and take the average output of these models as the final output if given an input image for classification. The adversarial examples are generated based on the start model $(t=0)$ or end model $(t=1)$ and we assume the attacker does not have any knowledge about the connection path nor the models on the path. We use these adversarial examples to test the performance of the model ensembling with the models on the connection. The attack success rate of adversarial examples can decrease from 85\% to 79\%. The defense improvement of this naive model ensembling strategy is not very significant, possibly due to the well-known transferrability of adversarial examples \citep{papernot2016transferability,su2018robustness}.  Similar findings can be concluded for the robust connection method.

\section{Experimental results on CIFAR-100}

To evaluate the proposed path connection method on more complicated image recognition benchmarks, we demonstrate its performance against backdoor and error-injection attacks on CIFAR-100 as shown in Figure \ref{fig: backdoor_injection_cifar100}. The experimental setting is similar to that of CIFAR-10, where the two end models are backdoored/error-injected, and the connection is trained with various bonafide data sizes. We can observe that our method is still able to remove the adversarial effects of backdooring or error-injection and repair the model. 

For the evasion attack, we investigate two cases, 1) the connection of two regular models and 2) the connection of regular and adversarially-trained models. The performance of adversarial training on CIFAR-100 is not as significant as that on CIFAR-10. So we do not investigate the connection of two adversarially-trained models.
Their loss and eigenvalue on the path are shown in Figure \ref{fig: input_hessian_vs_adversarial_loss_cifar100}. We can observe that there is also a high correlation between the robustness loss (loss of adversarial examples) and the largest eigenvalue of input Hessian.

\begin{figure}[h]   
 \centering
\begin{tabular}{p{1.9in}p{1.9in}p{1.3in}}
\includegraphics[align=c,width=1.9in]{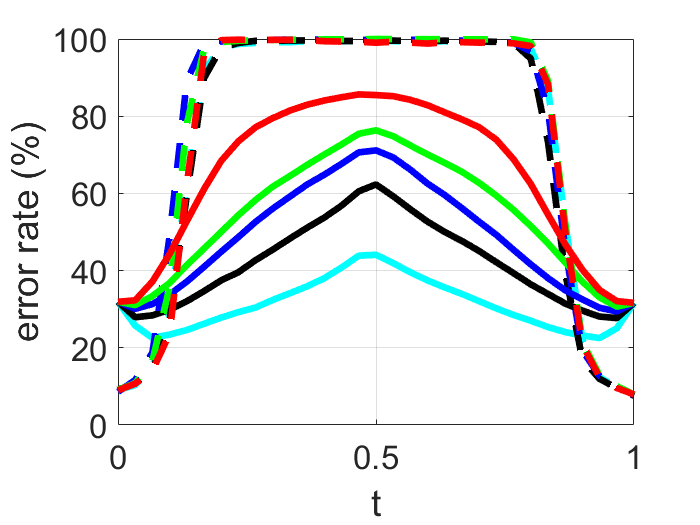}  &
\includegraphics[align=c,width=1.9in]{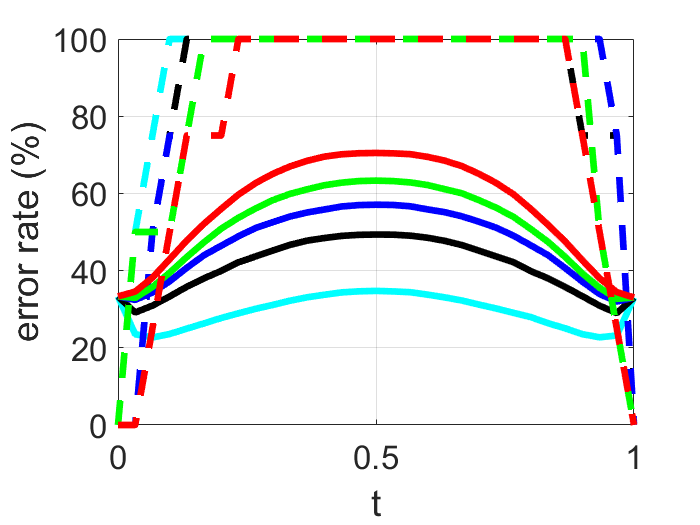} &
\includegraphics[align=c,width=1.3in]{figs_new/rebuttal/legend.png} 
\\
\parbox{1.9in}{\footnotesize \centering (a) Model connection against backdoor attack}
 & \parbox{1.9in}{\footnotesize \centering  (b) Model connection against error-injection attack} & \parbox{1.3in}{\footnotesize \centering  (c) Legend}
  \\
\end{tabular}
\caption{ Error rate against  poison and  error-injection attacks for CIFAR-100 (VGG).}
\label{fig: backdoor_injection_cifar100}
\end{figure}

\begin{figure}[t]    
 \centering
\begin{tabular}{p{0.1in}p{1.8in}p{1.8in}p{1.1in}}
 & \parbox{1.6in}{\centering \footnotesize Connection of regular models (PCC=0.74)} &  
\parbox{1.6in}{\centering \footnotesize  Connection of regular and adversarially-trained models (PCC=0.98)}    & 
\parbox{0.85in}{\centering \footnotesize Legend  } \\
\vspace{-0.4in} \rotatebox{90}{\parbox{0.6in}{\centering \footnotesize loss \& eigenvalue }}  &  \includegraphics[align=c,width=1.8in]{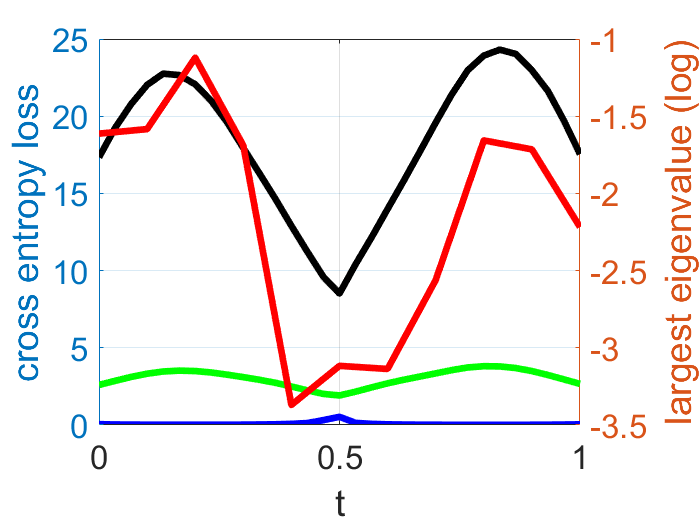} &
\includegraphics[align=c,width=1.8in]{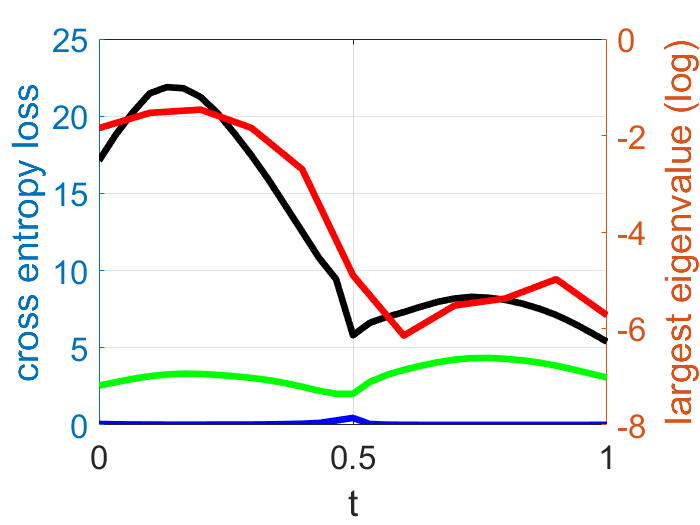} & 
 \includegraphics[align=c,width=1.1in]{figs_new/loss_legend.png} 
\\ 
\vspace{-0.5in} \rotatebox{90}{\parbox{0.9in}{\centering \footnotesize error rate \& \\ attack success rate }} & \includegraphics[align=c,width=1.8in]{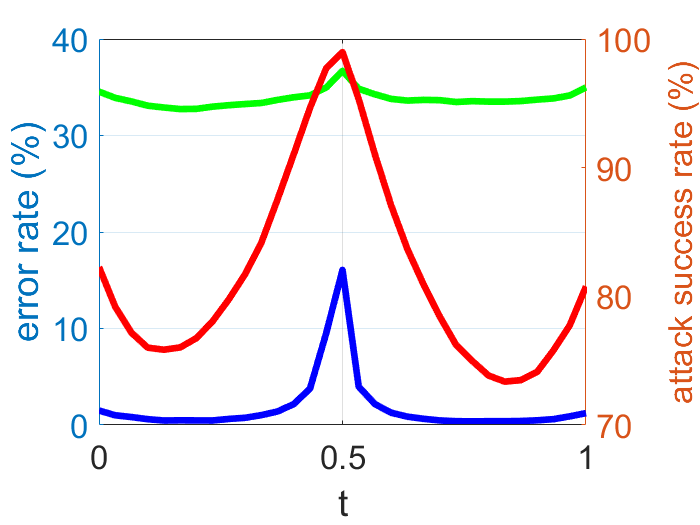} &
\includegraphics[align=c,width=1.8in]{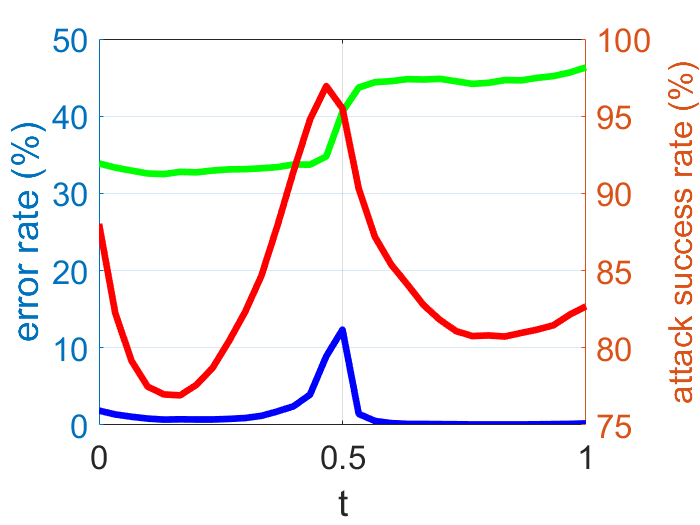} & 
\includegraphics[align=c,width=1.1in]{figs_new/error_legend.png}
\end{tabular}
\caption{Loss, error rate, attack success rate and largest eigenvalue of input Hessian on the path connecting different model pairs on CIFAR-100 (VGG) using standard loss. The error rate of training/test data means standard training/test error, respectively. In all cases, there is no standard loss barrier but a robustness loss barrier.
There is also a high correlation between the robustness loss and the largest eigenvalue of input Hessian, and their Pearson correlation coefficient (PCC) is reported in the title.} 
\label{fig: input_hessian_vs_adversarial_loss_cifar100}
\end{figure}

\section{Fine-tuning with various hyper-parameters}
We demonstrate the performance of fine-tuning with various hyper-parameter configurations in this section. For CIFAR-10 (VGG), we perform fine-tuning with different learning rate and the number of total epochs with the bonafide data of 2500 images and 1000 images, respectively. The clean accuracy and attack accuracy are shown in Figure \ref{fig: finetuning}. We also plot the clean and attack accuracy obtained through our path connection method from Table 2  in Figure \ref{fig: finetuning} as a reference.

As observed from Figure \ref{fig: finetuning} (a), larger learning rate (such as 0.05) can decrease the attack accuracy more rapidly, but the clean accuracy may suffer from a relatively large degradation. Small learning rate  (such as 0.01) can achieve high clean accuracy, but the attack accuracy may decrease with a much slower speed, leading to high attack accuracy when fine-tuning stops. This is more obvious if we use less bonafide data (reducing data size from 2500 to 1000) as shown in Figure \ref{fig: finetuning} (b). Fine-tuning performs worse with lower clean accuracy and higher attack accuracy. Since fine-tuning is quite sensitive to these hyper-parameters,
we conclude that it is not easy to choose an appropriate learning rate and the number of fine-tuning epochs, especially considering the user is not able to observe the attack accuracy in practice. On the other hand,
in  Figure \ref{fig: finetuning} (a) our path connection method can achieve the highest clean accuracy. In Figure \ref{fig: finetuning} (b), although the clean accuracy of $lr=0.01$ is higher than that of path connection, its attack accuracy remains high (about 40\%), which is much larger than that of path connection (close to 0\%). 

\begin{figure}[h]   
 \centering
\begin{tabular}{p{1.9in}p{1.9in}p{1.1in}}
\includegraphics[align=c,width=1.9in]{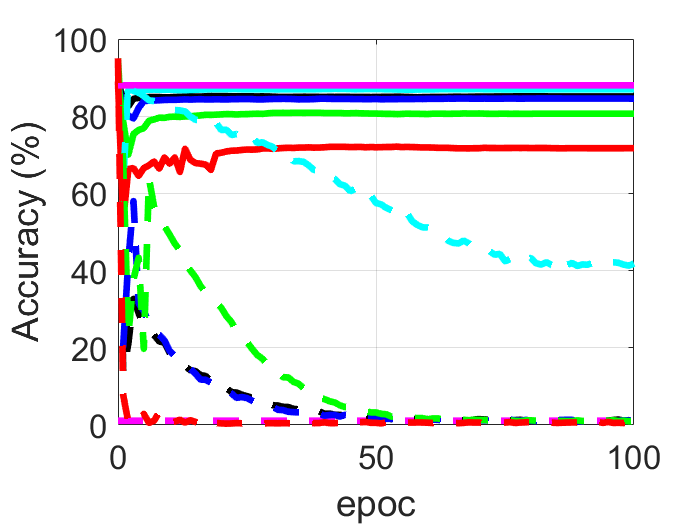}  &
\includegraphics[align=c,width=1.9in]{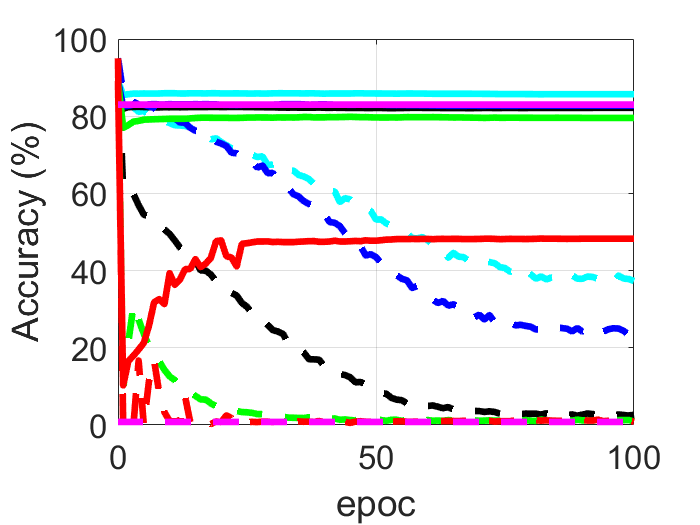} &
\includegraphics[align=c,width=1.1in]{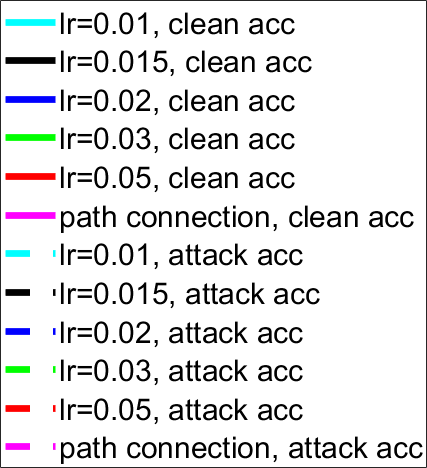} 
\\
\parbox{1.9in}{\footnotesize \centering (a) fine-tuning with  2500 samples}
 & \parbox{1.9in}{\footnotesize \centering  (b) fine-tuning with  1000 samples} & \parbox{1.1in}{\footnotesize \centering  (c) Legend}
  \\
\end{tabular}
\caption{ Test accuracy and attack accuracy for fine-tuning on CIFAR-10 (VGG).}
\label{fig: finetuning}
\end{figure}

\section{Stability Analysis}

In Appendix E, we perform multiple runs for the Gaussian noise experiment and only report the average accuracy. The variance is not reported since the average accuracy is able to demonstrate that the Gaussian noise method is not a good choice for removing adversarial effects. 

To investigate the stability of the path connection method with respect to every possible factor, it will cost considerable amount of time and resource to run all experiment setups considering the various attack methods, datasets, and model architectures. So here we mainly perform one representative experiment setup with multiple runs and show their mean and standard deviation. 

%\textcolor{blue}{Specify how the learning rate is varied.}

Figure \ref{fig: error_bars} shows the error bars of the error rate computed over 10 runs for path connection against  backdoor attack. The dataset is CIFAR-10 and the model architecture is ResNet. For each bonafide data size, we train 10 connections with different hyper-parameter settings, that is, starting from random initializations and using various learning rates (randomly set learning rate to 0.005, 0.01 or 0.02). 
%The value range of the test error and attack error over the 10 runs are plotted in Figure \ref{fig: error_bars}. 
Their average value and standard  deviation  are shown in Table \ref{tab: backdoor_comparison_error_bar}.  We can observe that although the connection may start from different initializations and trained with different learning rates,  their performance on the path are close with a relatively small variance, demonstrating the stability of our proposed method.

\begin{figure}[h]   
 \centering
\begin{tabular}{p{2.3in}p{2.3in}}
\includegraphics[align=c,width=1.9in]{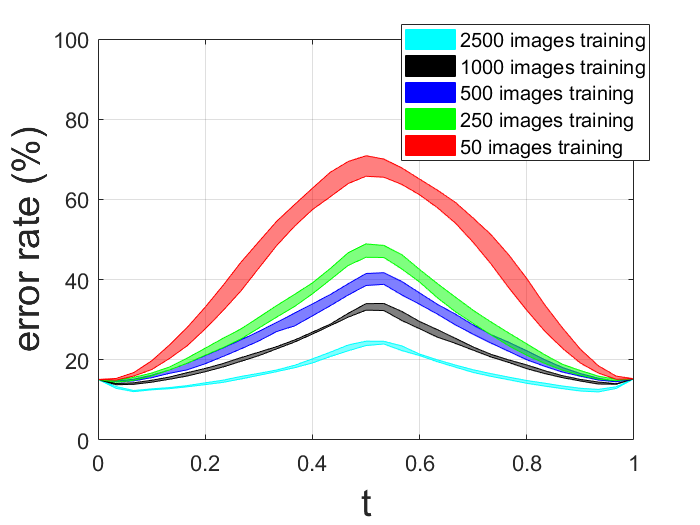}  &
\includegraphics[align=c,width=1.9in]{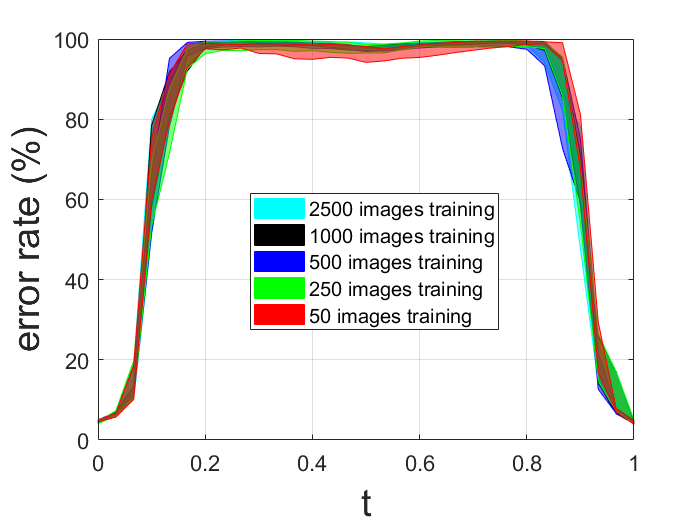} 
\\
\parbox{2.3in}{\footnotesize \centering (a) test error bar}
 & \parbox{2.3in}{\footnotesize \centering  (b) attack error bar} 
  \\
\end{tabular}
\caption{ Error rate against backdoor  attack for CIFAR-10 (ResNet).}
\label{fig: error_bars}
\end{figure}

\begin{table}[h]
\begin{center}
\caption{Performance comparison of path connection against single-target backdoor attack. The clean/backdoor accuracy means standard-test-accuracy/attack-success-rate, respectively.} 
\label{tab: backdoor_comparison_error_bar}
\scalebox{0.82}{
\begin{threeparttable}
\begin{tabular}{c|c|c|c|c|c|c|c}
\toprule[1pt]
& &  data size & 2500 & 1000  & 500  &  250 & 50\\
\midrule[1pt]
\multirow{ 8}{*}{\makecell{CIFAR-10 \\ (ResNet)}}   &\multirow{ 4}{*}{ \makecell{Clean \\  Accuracy}}&
$t=0.2$ & 87$\pm$0.21 \% &  83 $\pm$0.2 \% & 80.4 $\pm$0.72\%  & 78.4 $\pm$0.82\%  &  69 $\pm$1.9\% \\
& &$t=0.4$  &  80.5$\pm$0.37\% & 73.5$\pm$0.19\% &   67.7$\pm$1.0\%&  62.1$\pm$1.0\%&  40.5$\pm$2.5 \% \\
& &$t=0.6$  &  78.8$\pm$0.16\% & 71.6$\pm$0.7\% &   64.8$\pm$1.3 \%&  58.8$\pm$1.4\%&   36.9$\pm$1.4 \%\\
& & $t=0.8$  & 85.4 $\pm$0.25\% & 81.8$\pm$0.35\% &   78.8$\pm$0.84 \%& 76.8$\pm$1.1\%&  64$\pm$2.9\% \\
\cline{2-8}
&\multirow{ 4}{*}{ \makecell{Backdoor \\ Accuracy }}  & $t=0.2$  & 1.2$\pm$0.47 \% & 1.6$\pm$0.54\%  & 1.3$\pm$0.64 \%& 1.9$\pm$1\% & 2.6$\pm$0.54 \%  \\
&  & $t=0.4$   & 0.8$\pm$0.18 \% & 1.3$\pm$0.32\%  & 1.6$\pm$0.54\%  & 1.8$\pm$0.9\%  &3.2 $\pm$1.4\%\\
&  & $t=0.6$   & 1.0$\pm$0.16\%  & 1.5$\pm$0.41\%  & 1.6$\pm$0.42\%  & 1.6$\pm$0.57\%  &3.2$\pm$1.3\% \\
& &  $t=0.8$   & 0.8$\pm$0.37\%  & 0.86$\pm$0.33\%  & 1.0$\pm$0.78 \%  &  1.1$\pm$0.52\%  &1.8$\pm$0.3\% \\
\bottomrule[1pt]
\end{tabular}
\end{threeparttable}
}
\end{center}
\end{table}

\section{Strategy to choose the parameter $t$}
In our proposed method, we need to choose a model on the path as the  repaired model, that is, choosing the value of $t$. For different datasets/models, the best choice of $t$ may vary. So we discuss some general principles to choose an appropriate $t$ in this section. 

We note that in practice the user is not able to observe the attack error rate as the user does not have the knowledge about certain attacks. If the user is able to observe the accuracy on the whole clean test set, we suggest the user to choose the model (a value of $t\in[0,1]$) with a test accuracy $a-\Delta a$, where $a$ is the accuracy of the end model and $\Delta a$ represents a threshold. Based on the performance evaluations (Figures \ref{fig: backdoor_cifar_VGG}, \ref{fig: injection_cifar_VGG}, \ref{fig: backdoor_cifar_appendix}, \ref{fig: backdoor_svhn}, and \ref{fig: injection_cifar_res}), setting $\Delta a$ to 6\% should be an appropriate choice, which is able to eliminate the effects of all attacks without significantly sacrificing the clean accuracy. 

If the user is not able to access the accuracy of the whole clean test set, the user has to choose $t$ only based on the bonafide data. In this case, we suggest the user to use the $k$-fold cross-validation method to assess the test accuracy. This method  first shuffles the bonafide data randomly and splits it into $k$ groups. Then one group is kept to test the accuracy on the learned path and the remaining $k-1$ groups are used to train the path connection. The process is repeated for each group.
We perform additional experiments with the 5-fold cross validation method for CIFAR-10 (VGG) and  SVHN (ResNet). The average validation error rate on the hold-out set and attack error rate against backdoor attack is shown in Figure  \ref{fig: backdoor_injection_5fold}. The error-injection attack is easier to counter and hence we do not explore the error-injection experiments. 

We can observe that since the validation data size reduces to a much smaller value (one fifth of the bonafide data size), the test error rate becomes less stable. But it generally follows the trends of the test error rate on the whole clean test set (see Figure \ref{fig: backdoor_cifar_appendix}). So by utilizing the $k$-fold cross-validation method, we can obtain the test accuracy on the limited validation set. Then the aforementioned threshold method can be used to choose the model on the path. Here we suggest setting $\Delta a$ to 10\%, a more conservative threshold than the former case, as the user now does not has access to the accuracy of the whole test set. We also note that the performance of bonafide data with 50 images has a large deviation from other data size settings, so we suggest to use a larger bonafide data size with the k-fold cross validation method.

\begin{figure}[h]   
 \centering
\begin{tabular}{p{1.9in}p{1.9in}p{1.3in}}
\includegraphics[align=c,width=1.9in]{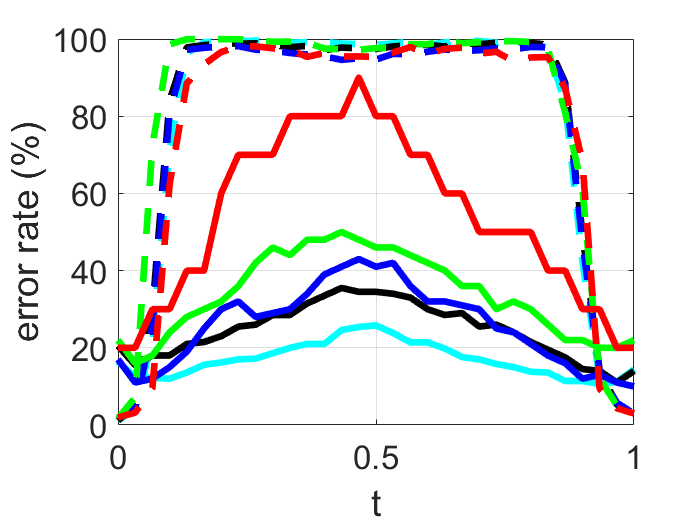}  &
\includegraphics[align=c,width=1.9in]{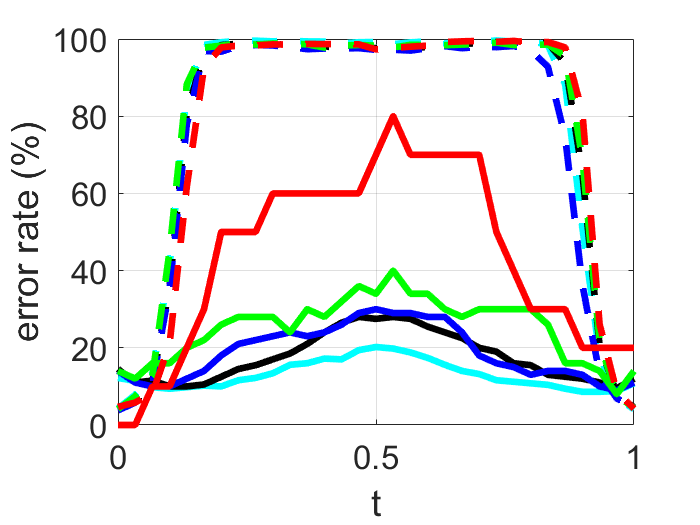}  &
\includegraphics[align=c,width=1.3in]{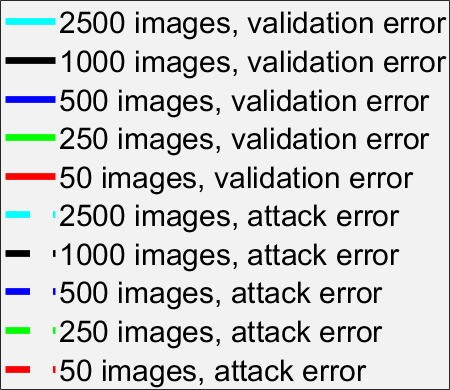} 
\\
\parbox{1.9in}{\footnotesize \centering (a) Model connection for CIFAR-10 (VGG)} & 
\parbox{1.9in}{\footnotesize \centering (b) Model connection for SVHN (ResNet)} & 
\parbox{1.3in}{\footnotesize \centering  (c) Legend}
  \\
\end{tabular}
\caption{ Average 5-fold cross validation error rate and attack error rate against  backdoor attack.}
\label{fig: backdoor_injection_5fold}
\end{figure}

\end{document}